\documentclass[final,12pt]{colt2025} % Include author names

\title[Agnostic universal rates of ERM]{Universal rates of ERM for agnostic learning}
\usepackage{times}

\usepackage{bm,bbm}
\usepackage{enumerate}
\usepackage{enumitem}
\usepackage{mathrsfs}
\usepackage{mathtools}
\usepackage{cite}
\usepackage{mdframed}
\usepackage{eqparbox}
\usepackage[all]{xy}
\usepackage{caption}
\usepackage{float}

\DeclareMathOperator*{\argmin}{arg\,min}

\newtheorem{condition}{Condition}
\newtheorem{question}{Question}

\newcommand\naturalnumber{\mathbb{N}}
\newcommand\R{\mathbb{R}}
\newcommand\E{\mathbb{E}}

\newcommand\hpc{\mathcal{H}}
\newcommand\lalgo{\hat{h}_{n}}

\newcommand\trueerrorrateh{\text{er}_{P}(h)}
\newcommand\trueerrorratehn{\text{er}_{P}(\hat{h}_{n})}
\newcommand\trueerrorratehstar{\text{er}_{P}(h^{*})}
\newcommand\empiricalerrorrateh{\hat{\text{er}}_{S_{n}}(h)}
\newcommand\empiricalerrorratehn{\hat{\text{er}}_{S_{n}}(\hat{h}_{n})}
\newcommand\empiricalerrorratehstar{\hat{\text{er}}_{S_{n}}(h^{*})}

\newcommand\excessrisk{\mathcal{E}(n,P)}
\newcommand\bayesoptimal{h_{\text{Bayes}}^{*}}
\newcommand\datasetstats{S_{n}:=\{(x_{i}, y_{i})\}_{i=1}^{n}\sim P^{n}}
\newcommand\vcdim{\text{VC}(\mathcal{H})}

\newcommand\starnumber{\mathfrak{s}}
\newcommand\regionofdisagreement{\text{DIS}(\mathcal{H})}

\def\ddefloop#1{\ifx\ddefloop#1\else\ddef{#1}\expandafter\ddefloop\fi}
\def\ddef#1{\expandafter\def\csname v#1\endcsname{\ensuremath{\boldsymbol{#1}}}}
\ddefloop abcdefghijklmnopqrstuvwxyzABCDEFGHIJKLMNOPQRSTUVWXYZ\ddefloop
\def\ddef#1{\expandafter\def\csname v#1\endcsname{\ensuremath{\boldsymbol{\csname #1\endcsname}}}}
\ddefloop {alpha}{beta}{gamma}{delta}{epsilon}{varepsilon}{zeta}{eta}{theta}{var
theta}{iota}{kappa}{lambda}{mu}{nu}{xi}{pi}{varpi}{rho}{varrho}{sigma}{varsigma}
{tau}{upsilon}{phi}{varphi}{chi}{psi}{omega}{Gamma}{Delta}{Theta}{Lambda}{Xi}{Pi
}{Sigma}{varSigma}{Upsilon}{Phi}{Psi}{Omega}{ell}\ddefloop

\coltauthor{%
 \Name{Steve Hanneke} \Email{steve.hanneke@gmail.com}\\
 \addr Purdue University
 \AND
 \Name{Mingyue Xu} \Email{xu1864@purdue.edu}\\
 \addr Purdue University%
}

\begin{document}

\maketitle

\begin{abstract}%
    The universal learning framework has been developed to obtain guarantees on the learning rates that hold for any fixed distribution, which can be much faster than the ones uniformly hold over all the distributions. Given that the \emph{Empirical Risk Minimization} (ERM) principle being fundamental in the PAC theory and ubiquitous in practical machine learning, the recent work of \citet{hanneke2024universal} studied the universal rates of ERM for binary classification under the realizable setting. However, the assumption of realizability is too restrictive to hold in practice. Indeed, the majority of the literature on universal learning has focused on the realizable case, leaving the non-realizable case barely explored.
    
    In this paper, we consider the problem of universal learning by ERM for binary classification under the agnostic setting, where the ``learning curve" reflects the decay of the excess risk as the sample size increases. We explore the possibilities of agnostic universal rates and reveal a compact \emph{trichotomy}: there are three possible agnostic universal rates of ERM, being either $e^{-n}$, $o(n^{-1/2})$, or arbitrarily slow. We provide a complete characterization of which concept classes fall into each of these categories. Moreover, we also establish complete characterizations for the target-dependent universal rates as well as the Bayes-dependent universal rates.
\end{abstract}


\section{Introduction}
  \label{sec:introduction}
The classic PAC (Probably Approximately Correct) theory \citep{vapnik1974theory,valiant1984theory} focuses on understanding the best worst-case (\emph{uniform}) learning rate by a learning algorithm over all data distributions. Due to its distribution-free nature, the PAC framework fails to capture the distribution-dependent rates of learning hypothesis classes, which are possibly faster than the uniform learning rates \citep{cohn1990can,cohn1992tight}. From a practical perspective, the distribution for data generation is typically fixed in real-world learning problem and the collected data is rarely worst-case, the PAC framework is therefore too pessimistic to explain practical machine learning performance. Universal learning, a distribution-dependent framework that helps to understand machine learning beyond the classic PAC setting, has been proposed by \citet{bousquet2021theory} and actively studied recently \citep{bousquet2023fine,hanneke2022universal,hanneke2023universal,attias2024universal,hanneke2024universal,hanneke2024active}. The universal learning model adopts a setting where the data distribution is fixed and the performance of a learning algorithm is measured by its ``learning curve", i.e., the decay of the expected error as a function of the input sample size, and such rate of decay is called a \emph{universal} rate. Indeed, in the work of \citet{bousquet2021theory}, they showed that for binary classification in the realizable setting, the optimal universal rates are captured by a \emph{trichotomy}: every concept class $\hpc$ has a universal rate being either exponential, linear or arbitrarily slow. Compared to a well-known \emph{dichotomy} of the optimal uniform rates: every concept class $\hpc$ has a uniform rate being either linear $\vcdim/n$ or “bounded away from zero", this makes an impression that universal rates may differ a lot from the uniform rates.

In supervised learning, the celebrated \emph{empirical risk minimization} (ERM) principle \citep{vapnik1998statistical} stands at the center of many successful learning algorithms that seek to minimize the average error over the training data. In practice, ERM-based algorithms have been innovatively designed and widely applied in different areas of machine learning. For example, most successful applications of deep neural networks in fields such as computer vision \citep{krizhevsky2012imagenet}, speech recognition \citep{hinton2012deep}, and reinforcement learning \citep{mnih2015human} have their models trained to minimize the empirical error, leveraging those renowned optimization algorithms such as GD, SGD and Adam \citep{KingBa15}. In learning theory, the ERM principle has also been shown to have fundamental importance in understanding the PAC learnability: a concept class is learnable if and only if it can be learned by any ERM algorithm. 

While the role of ERM in the classic PAC theory has been very well understood, the topics of universal learning by ERM have remained underexplored. The recent work of \citep{hanneke2024universal} studied the universal rates of ERM for binary classification problem in the realizable setting. They showed that the universal rates of ERM are captured by a \emph{tetrachotomy}: every concept class that is learnable by ERM has a universal rate being either $e^{-n}$, $1/n$, $\log(n)/n$, or arbitrarily slow. The realizable case is indeed an idealistic scenario where a perfect hypothesis is assumed to exist, i.e., $\inf_{h\in\hpc}\trueerrorrateh=0$. However, in real-world machine learning applications, the ground-truth models are often complicated and unknown to practitioners. These considerations motivate us to study the universal rates of ERM in the agnostic setting, a more realistic and applicable situation where the true concept may not be in the hypothesis class, i.e., $\inf_{h\in\hpc}\trueerrorrateh>0$, and the goal is to find a hypothesis being competitive with the best hypothesis (in class). In this paper, we aim to answer the following fundamental question:
\begin{question}
  \label{ques:fundamental-question}
Given a concept class $\hpc$, what are the possible rates at which $\hpc$ can be agnostically universally learned by ERM?  
\end{question}

\subsection{Notations and preliminaries}
  \label{subsec:notations-and-preliminaries}
Following the classical setup of statistical learning, we consider a binary classification problem with an instance space $\mathcal{X}$ and a concept class $\hpc\subseteq\{0,1\}^{\mathcal{X}}$. Let $h:\mathcal{X}\rightarrow\{0,1\}$ be a classifier. Given a probability distribution $P$ on $\mathcal{X}\times\{0,1\}$, the \emph{error rate} of $h$ is defined as $\trueerrorrateh:=P((x,y)\in\mathcal{X}\times\{0,1\}:h(x)\neq y)$. Given a dataset $S_{n}:=\{(x_{i},y_{i})\}_{i=1}^{n}\in(\mathcal{X}\times\{0,1\})^{n}$, the \emph{empirical error rate} of $h$ is defined as $\empiricalerrorrateh:=\frac{1}{n}\sum_{i=1}^{n}\mathbbm{1}(h(x_{i})\neq y_{i})$. Let $P$ be a data distribution, for an integer $n$, we denote by $\datasetstats$ an i.i.d. $P$-distributed dataset. Recall that a distribution $P$ is called realizable with respect to $\hpc$ if it satisfies $\inf_{h\in\hpc}\trueerrorrateh=0$. Note that for realizable learning, an ERM learner is any learning algorithm that outputs a sample-consistent classifier, that is, a classifier in the sample-induced version space \citep{mitchell1977version}. In this paper, we consider instead the often more realistic setting of agnostic learning, where $\inf_{h\in\hpc}\trueerrorrateh>0$. In the agnostic setting, an ERM algorithm is any learning algorithm outputs the hypothesis achieving the best performance on the training data (breaking ties arbitrarily), i.e., $\lalgo=\text{ERM}(S_{n}):=\argmin_{h\in\hpc}\empiricalerrorrateh$. For simplicity, we conflate the ERM learner $\lalgo$ with the hypothesis it returns throughout the paper.

In the realizable setting, the PAC learning aims to achieve $\trueerrorratehn\leq\epsilon$ for the error $\epsilon$ going to 0 as fast as possible with $n$, and the universal learning focuses on the rate of decay of the so-called \emph{learning curve}, that is, the decay of the expected error rate $\E[\trueerrorratehn]$ as a function of sample size $n$. In the agnostic setting, the goal of PAC learning is instead to guarantee that the excess risk $\trueerrorratehn-\inf_{h\in\hpc}\trueerrorrateh\leq\epsilon$ for $\epsilon$ going to 0 as fast as possible with $n$. Therefore, for universal learning, it is natural to extend the notion of learning curve as the decay of the expected excess risk as a function of sample size $n$. Concretely, we define the expected excess risk as follow.

\begin{definition}  [\textbf{Excess risk}]
  \label{def:excess-risk}
Let $\hpc$ be a concept class, and let $\{\lalgo\}_{n\in\naturalnumber}$ be the output of an ERM algorithm. For any distribution $P$ over $\mathcal{X}\times\{0,1\}$ and data $\datasetstats$, we define its (expected) \underline{excess risk} as
\begin{equation*}
    \excessrisk := \E\left[\trueerrorratehn - \inf_{h\in\hpc}\trueerrorrateh\right] .
\end{equation*}
Moreover, we say that a distribution \underline{$P$ is centered at $h^{*}$} for some $h^{*}:\mathcal{X}\rightarrow\{0,1\}$ if it satisfies $\trueerrorratehstar=\inf_{h\in\hpc}\trueerrorrateh$ and also $\inf_{h\in\hpc}P_{\mathcal{X}}(x:h(x)\neq h^{*}(x))=0$, and then $h^{*}$ is called a \underline{target function} of the learning problem. 
\end{definition}
We underline that a target function may not be in the concept class which is standard for agnostic learning, and a distribution $P$ can have multiple target functions (Example~\ref{ex:no-infinite-eulder-sequence-but-no-faster-than-o(n^(-1/2))-rates}). With these settings settled, we are now able to define the problem of agnostic universal learning by ERM. Following \citet{hanneke2024universal}, we extend the definition from the realizable case. 
\begin{definition}  [\textbf{Agnostic universal learning by ERM}]
  \label{def:agnostic-universally-rates}
    Let $\hpc$ be a concept class and $R(n)\rightarrow 0$ be a rate function. We say
    \setlist{nolistsep}
    \begin{itemize}
        \item $\hpc$ is \underline{agnostically universally learnable at rate $R$ by ERM}, if for every distribution $P$, there exist parameters $C, c > 0$ such that for every ERM algorithm, its excess risk satisfies $\excessrisk \leq CR(cn)$, for all $n\in\naturalnumber$.
        \item $\hpc$ is \underline{not agnostically universally learnable at rate faster than $R$ by ERM}, if there exists a distribution $P$ and parameters $C, c > 0$ such that there is an ERM algorithm satisfying $\excessrisk \geq CR(cn)$, for infinitely many $n\in\naturalnumber$.
        \item $\hpc$ is \underline{agnostically universally learnable with exact rate $R$ by ERM}, if $\hpc$ is agnostically universally learnable at rate $R$ by ERM, and is not agnostically universally learnable at rate faster than $R$ by ERM.
        \item $\hpc$ requires \underline{at least arbitrarily slow rates to be agnostically universally learned by ERM}, if for any $R(n)\rightarrow 0$, $\hpc$ is not agnostically universally learnable at rate faster than $R$ by ERM. 
    \end{itemize}
\end{definition}
We emphasize that a crucial difference between Definition~\ref{def:agnostic-universally-rates} and the PAC learning is that here the constants $C,c>0$ are allowed to depend on the distribution $P$. Moreover, it is straightforward from the definition that we are basically considering the worst-case ERM here. The following extensions are required for presenting our results.
\begin{definition}
  \label{def:asymptotic-agnostic-universally-rates}
    Following the notations in Definition \ref{def:agnostic-universally-rates}, we say that $\hpc$ is
    \setlist{nolistsep}
    \begin{itemize}
        \item \underline{agnostically universally learnable at rate $o(n^{-1/2})$ by ERM}, if for every distribution $P$ and every ERM algorithm, $\excessrisk=o(n^{-1/2})$, for all $n\in\naturalnumber$.
        \item \underline{not agnostically universally learnable at rate faster than $o(n^{-1/2})$ by ERM}, if for any $T(n)=o(n^{-1/2})$, there exists a distribution $P$ such that there is an ERM algorithm satisfying $\excessrisk \geq T(n)$, for infinitely many $n\in\naturalnumber$.
        \item \underline{agnostically universally learnable with exact rate $o(n^{-1/2})$ by ERM}, if the above two hold.
    \end{itemize}
    Here, $o$ is the standard asymptotic notation that can be distribution-dependent when $n\to\infty$. 
\end{definition}
\begin{definition}
  \label{def:agnostic-universally-rates-under-distributions}
    For a class of distributions $\mathcal{P}$, \underline{the agnostic universal learning of $\hpc$ under $\mathcal{P}$ by ERM} is defined as the same as Definition~\ref{def:agnostic-universally-rates} except considering only distributions $P\in\mathcal{P}$ instead of all the probability distributions $P$ over $\mathcal{X}\times\{0,1\}$.
\end{definition}

\subsection{Related works}
  \label{subsec:related-works} 
\textbf{PAC learning by ERM.} The performance of ERM algorithms in the PAC framework has been well understood. For the realizable case, the optimal sample complexity of ERM learners \citep{blumer1989learnability,vapnik1974theory} is $\mathcal{M}_{\text{ERM}}^{\hpc}(\epsilon,\delta)=\Theta((\vcdim\log{(1/\epsilon)}+\log{(1/\delta)})/\epsilon)$, resulting in the uniform rate $\trueerrorratehn=\Theta((\vcdim\log{(n/\vcdim)}+\log{(1/\delta)})/n)$, which is sub-optimal due to an unavoidable logarithmic factor \citep{auer2007new}. Indeed, it has been proved that this uniform rate is the best achievable for any proper learner \citep{haussler1994predicting,simon2015almost}, whereas there are improper learners can achieve a rate of $\Theta((\vcdim+\log{(1/\delta)})/n)$ 
\citep{hanneke2016optimal,aden2023optimal,larsen2023bagging,aden2024majority}. However, for the agnostic case, ERM learners have the optimal sample complexity $\mathcal{M}_{\text{ERM}}^{\hpc,\text{AG}}(\epsilon,\delta)=\Theta((\vcdim+\log{(1/\delta)})/\epsilon^{2})$, and thus guaranteeing $\trueerrorratehn-\inf_{h\in\hpc}\trueerrorrateh=\Theta(\sqrt{(\vcdim+\log{(1/\delta)})/n})$ \citep{haussler1992decision}, which is optimal for any learning algorithm including improper learners. It is worth mentioning that this discrepancy of the optimality of the ERM rule between the two settings has been studied recently in the work of \citet{hanneke2024revisiting}, where they showed that ERM is indeed sub-optimal when treating $\inf_{h\in\hpc}\trueerrorrateh$ as a parameter of the rates.

\noindent\textbf{Universal learning.} While the standard PAC model has been dominating the learning theory, the fact that practical learning rates can be much faster than the one described in the PAC theory, was not only observed from empirical experiments \citep{cohn1990can,cohn1992tight} but also verified by some early theoretical works \citep{schuurmans1997characterizing,koltchinskii2005exponential,audibert2007fast,pillaud2018exponential}, where exponentially fast learning rates were guaranteed under specific model assumptions (e.g., for kernel methods and stochastic gradient decent). These findings motivate the development of alternative learning models that can better help to understand the practice of machine learning. The property of \emph{universal consistency} was first provided by \citet{stone1977consistent} and later generalized by \citet{hanneke2021learning}, establishing the existence of universally consistent learning algorithms in any separable metric space. The work of \citet{benedek1988nonuniform} considered a relaxation of the PAC model that lies in between the uniform and the universal settings called \emph{nonuniform learning}, where the learning rate may depend on the target concept but still uniform over marginal distributions. The work of \citet{van2013universal} studied the uniform convergence property from a universal perspective and gave out a combinatorial characterization of the \emph{universal Glivenko-Cantelli} property (Definition~\ref{def:universal-glivenko-cantelli}). Until very recently, the universal learning framework was first formalized by \citet{bousquet2021theory}, along with a complete theory of the optimal universal rates. After that, \citet{bousquet2023fine} carried out a fine-grained analysis on the ``distribution-free tail" of the universal \emph{learning curves} by characterizing the optimal constant factor. As generalizations, \citet{kalavasis2022multiclass,hanneke2023universal,hanneke2022universal,hanneke2024active} studied the universal rates for other settings including multiclass classification, active learning, interactive learning, etc. The most relevant work to ours is \citet{hanneke2024universal}, who studied the universal rates of ERM for binary classification problem in the realizable setting.

\section{Main Results}
  \label{sec:main-results}
In this section, we summarize the main results of this paper as well as the related technical notions of complexity. In brief, we study both target-independent and target-dependent agnostic universal rates by ERM. Moreover, since the target-dependent result relies on certain ad-hoc conditions which are lacking of intuition, we further propose to categorize a data distribution according to its Bayes-optimal classifier (Definition \ref{def:bayes-optimal-classifier}) and show that the corresponding Bayes-dependent universal rates are characterized by some simple combinatorial structures. Further details for these results will be discussed in the following Sections \ref{sec:target-inpdependent-rates}-\ref{sec:bayes-dependent-rates}, along with related technical analyses and proof sketches. 

We start with target-independent agnostic universal rates. We reveal a fundamental \emph{trichotomy} in the following Theorem \ref{thm:main-theorem-target-independent}, namely there are exactly three possibilities for the agnostic universal rates by ERM: being either exponential ($e^{-n}$), or super-root ($o(n^{-1/2})$), or at least arbitrarily slow. Moreover, the characterization (that determines for each concept class which of the three categories it belongs to) simply consists of the cardinality and the VC dimension of the concept class.
\begin{theorem}  [\textbf{Agnostic universal rates for ERM}]
  \label{thm:main-theorem-target-independent}
For every concept class $\hpc$ with $|\hpc|\geq 3$, 
\setlist{nolistsep}
\begin{itemize}
    \item $\hpc$ is agnostically universally learnable by ERM with exact rate $e^{-n}$ if and only if $|\hpc|<\infty$.
    \item $\hpc$ is agnostically universally learnable by ERM with exact rate $o(n^{-1/2})$ if and only if $|\hpc|=\infty$ and $\vcdim<\infty$.
    \item $\hpc$ requires at least arbitrarily slow rates to be agnostically universally learned by ERM if and only if $\vcdim=\infty$.
\end{itemize}
\end{theorem}
It is worthwhile to mention the following technical aspects of the proof. Firstly, to show the upper bound of super-root rates $o(n^{-1/2})$, we prove a refined version of a classic \emph{uniform Bernstein} inequality (Proposition~\ref{prop:uniform-bernstein}) which improves a result of \citet{vapnik1974theory} by a logarithmic factor. In summary, its proof applies a combination of localization \citep{bartlett04,bartlett05,koltchinskii06}, a concentration inequality \citep{bousquet2002bennett} and an entropy integral bound on the rate of uniform convergence \citep{van-der-Vaart96,gine06,van-der-Vaart11} accounting for variances of loss differences, together with well-known bounds on the covering numbers of VC classes \citep{haussler95}. Secondly, to get the lower bounds of $o(n^{-1/2})$ and arbitrarily slow rates, the techniques are quite different from proving ERM lower bounds in the classic PAC theory. Concretely, the idea is to use the following equivalences:
\begin{lemma}  [{\textbf{\citealp[][Lemma 8]{hanneke2024universal}}}]
  \label{lem:equivalence-infinite-eluder-and-infinite-cardinality}
Any concept class $\hpc$ has an infinite eluder sequence (Definition~\ref{def:eluder-sequence}) if and only if $|\hpc|=\infty$.
\end{lemma}

\begin{lemma}  [{\textbf{\citealp[][Lemma 9]{hanneke2024universal}}}]
  \label{lem:equivalence-infinite-vc-eluder-and-infinite-vcdim}
Any concept class $\hpc$ has an infinite VC-eluder sequence (Definition~\ref{def:vc-eluder-sequence}) if and only if $\vcdim=\infty$.
\end{lemma}
We construct skillfully-designed distributions on such infinite sequences to support the lower bounds, unlike in the PAC theory such distributions for lower bounds are often constructed on finite sets.
% \begin{table}[H]
% \begin{tabular}{|p{1.9cm}|p{2.6cm}|p{7.7cm}|p{1.3cm}|}
%  \hline
%  \small Optimal rates & \small (ERM) exact rates & \small Cases & \small Examples \\
%  \hline
%  $e^{-o(n)}$ & $o(n^{-1/2})$ & \text{\small $\exists$ infinite eluder sequence, $\nexists$ infinite Littlestone tree} & Ex.12 \\
%  \hline
%  $e^{-o(n)}$ & arbitrarily slow & \text{\small $\exists$ infinite VC-eluder sequence, $\nexists$ infinite Littlestone tree} & Ex.15 \\
%  \hline
%  $o(n^{-1/2})$ & arbitrarily slow & \text{\small $\exists$ infinite VC-eluder sequence, $\nexists$ infinite VCL tree} & Ex.16 \\
%  \hline
% \end{tabular}
% \caption*{Discrepancy between the optimal agnostic universal rates and agnostic universal rates of ERM (this paper). Examples can be found in \citet{hanneke2024universal} with the same numbers.}
% \end{table}
\begin{table}
\begin{tabular}{|p{2.2cm}|p{2.2cm}|p{6.9cm}|p{2.2cm}|}
 \hline
 \small Realizable case & \small Agnostic case & \small Categories & \small Examples \\
 \hline
 $e^{-n}$ & $e^{-n}$ & \text{\small $|\hpc|<\infty$} & \text{\small finite class} \\
 \hline
 $1/n$ & $o(n^{-1/2})$ & \text{\small $|\hpc|=\infty$, $\vcdim<\infty$, $\nexists$ infinite star-eluder seq} & \text{\small thresholds on $\naturalnumber$} \\
 \hline
 $\log(n)/n$ & $o(n^{-1/2})$ & \text{\small $\vcdim<\infty$, $\exists$ infinite star-eluder seq} & \text{\small singletons on $\naturalnumber$} \\
 \hline
 \text{\small arbitrarily slow} & \text{\small arbitrarily slow} & \text{\small $\vcdim=\infty$} & \text{\small non-VC class} \\
 \hline
\end{tabular}
\caption{Comparison of the ERM universal rates between the realizable case and the agnostic case. The definition to a star-eluder sequence can be found in Appendix~\ref{sec:extra-definitions}.}
\end{table}
% \vspace{-1.5em}
\begin{table}
\begin{tabular}{|p{3.6cm}|p{5.0cm}|p{5.4cm}|}
 \hline
 \small Uniform rates & \small Universal rates & \small Categories \\
 \hline
 $n^{-1/2}$ & $e^{-n}$ & \text{\small $|\hpc|<\infty$} \\
 \hline
 $n^{-1/2}$ & $o(n^{-1/2})$ & \text{\small $|\hpc|=\infty$, $\vcdim<\infty$} \\
 \hline
 \text{\small $\Omega(1)$ (not learnable)} & \text{\small at least arbitrarily slow} & \text{\small $\vcdim=\infty$} \\
 \hline
\end{tabular}
\caption{Comparison between the agnostic universal rates and the agnostic uniform rates of ERM.}
\end{table}

Before proceeding to the target-dependent rates, we first introduce some relevant definitions and technical assumptions. Throughout this paper, we will often assume that a concept class satisfies the universal Glivenko-Cantelli property.
\begin{definition}  [{\textbf{\citealp[Glivenko-Cantelli class,][]{van2013universal}}}]
  \label{def:universal-glivenko-cantelli}
Let $\hpc$ be a concept class on an instance space $\mathcal{X}$. Given a probability distribution $P$ on $\mathcal{X}\times\{0,1\}$, let $\{(X_{i},Y_{i})\}_{i\geq1}$ be a sequence of independently $P$-distributed random samples. We say that $\hpc$ is a \underline{$P$-Glivenko-Cantelli} class if
\begin{equation*}
    \sup_{h\in\hpc}\Big|\empiricalerrorrateh-\trueerrorrateh\Big| \overset{p}{\longrightarrow}0, \;\text{ as } n\rightarrow\infty ,
\end{equation*}
where the convergence rate can be $P$-dependent. We say $\hpc$ is a \underline{universal Glivenko-Cantelli (UGC)} class if it is $P$-Glivenko-Cantelli for every distribution $P$.
\end{definition}
We remark that while a finite VC dimension is neither sufficient nor necessary to ensure the universal Glivenko-Cantelli property of a hypothesis class, a bunch of works \citep{vapnik1971uniform,dudley1991uniform,van2000preservation} have shown that under weak measurability conditions (e.g., image-admissible Suslin, universal separability), a finite VC dimension is in fact equivalent to the uniform Glivenko-Cantelli property, which of course implies the universal Glivenko–Cantelli property. We then introduce two technical target-dependent conditions.
\begin{condition}
  \label{cond:constant-gap-error-rate}
For any distribution $P$ centered at $h^{*}$, the following holds 
\begin{equation*}
    \inf_{h\in\hpc:\trueerrorrateh>\trueerrorratehstar}\left\{\trueerrorrateh-\trueerrorratehstar\right\} > 0 . \;\; (\text{define } \inf_{\emptyset}\{\cdot\}=1)
\end{equation*}
\end{condition}

\begin{condition}
  \label{cond:finite-vcd-for-sufficiently-small-ball}
For any distribution $P$ centered at $h^{*}$, there exists $\epsilon_{0}:=\epsilon_{0}(P)>0$ such that
\begin{equation*}
    \text{VC}\left(\left\{h\in\hpc: 0<\text{er}_{P}(h)-\trueerrorratehstar\leq\epsilon_{0}\right\}\right) < \infty .
\end{equation*}
\end{condition}
Further discussions about these conditions can be found in Section \ref{sec:target-dependent-rates}. Let us present a \emph{trichotomy} capturing the target-dependent agnostic universal rates by ERM. 
\begin{theorem}  [\textbf{Target-dependent agnostic universal rates}]
  \label{thm:main-theorem-target-dependent}
For every UGC class $\hpc$ with $|\hpc|\geq 3$ and every classifier $h^{*}$, let $\mathcal{P}_{h^{*}}$ be the set of all distributions centered at $h^{*}$, then the following hold:
\setlist{nolistsep}
\begin{itemize}
    \item $\hpc$ is agnostically universally learnable under $\mathcal{P}_{h^{*}}$ by ERM with exact rate $e^{-n}$ if and only if Condition \ref{cond:constant-gap-error-rate} holds for $h^{*}$.
    \item $\hpc$ is agnostically universally learnable under $\mathcal{P}_{h^{*}}$ by ERM with exact rate $o(n^{-1/2})$ if and only if Condition \ref{cond:constant-gap-error-rate} fails and Condition \ref{cond:finite-vcd-for-sufficiently-small-ball} holds for $h^{*}$.
    \item $\hpc$ requires at least arbitrarily slow rates to be agnostically universally learned under $\mathcal{P}_{h^{*}}$ by ERM if and only if Condition \ref{cond:finite-vcd-for-sufficiently-small-ball} fails for $h^{*}$.
\end{itemize}
\end{theorem}
We notice that, though providing an ``if and only if" characterization to the target-dependent universal rates, Conditions~\ref{cond:constant-gap-error-rate}, \ref{cond:finite-vcd-for-sufficiently-small-ball} are not based on simple combinatorial structures, and thus not broadly applicable to general concept classes. A naturally follow-up question is that whether there is a better function-specified universal rates result with a complete characterization based on some combinatorial structures. To address this limitation, we consider to categorize a distribution according to its Bayes-optimal classifier.
\begin{definition}  [\textbf{Bayes-optimal classifier}]
  \label{def:bayes-optimal-classifier}
A \underline{Bayes-optimal classifier} with respect to a distribution $P$, denoted by $\bayesoptimal$, is defined to be a binary function such that $\text{er}_{P}(\bayesoptimal) = \inf_{h:\mathcal{X}\rightarrow\{0,1\}}\trueerrorrateh$. Moreover, let $\eta(x;P):=P(Y=1|X=x)$, then $\bayesoptimal(x)=\mathbbm{1}(\eta(x;P)\geq 1/2)$, for all $x\in\mathcal{X}$.
\end{definition}
The following sequential structures developed by \citet{hanneke2024universal} formalize the characterization of the Bayes-dependent agnostic universal rates.
\begin{definition}  [\textbf{Eluder sequence}]
  \label{def:eluder-sequence}
A (finite or infinite) data sequence $\{(x_{1}, y_{1}),(x_{2}, y_{2}),\ldots\}\in(\mathcal{X}\times\{0,1\})^{\infty}$ is called \underline{realizable} (with respect to $\hpc$) if for every $n\in\naturalnumber$, there exists $h_{n}\in\hpc$ such that $h_{n}(x_{i})=y_{i}$, for all $i\in[n]$. Let $h$ be a classifier, we say that $\hpc$ has an infinite \underline{eluder sequence} $\{(x_{1}, y_{1}),(x_{2}, y_{2}),\ldots\}$ \underline{centered at $h$} if it is realizable and labelled by $h$, and for every integer $k\geq1$, there exists $h_{k}\in\hpc$ such that $h_{k}(x_{i})=y_{i}$ for all $i<k$ and $ h_{k}(x_{k}) \neq y_{k}$.
\end{definition}

\begin{definition}  [\textbf{VC-eluder sequence}]
  \label{def:vc-eluder-sequence}
Let $S_{n}:=\{(x_{i},y_{i})\}_{i=1}^{n}$ be a dataset, the \underline{version space} (induced by $S_{n}$), denoted by $V_{S_{n}}(\hpc)$ (or $V_{n}(\hpc)$), is defined as $V_{S_{n}}(\hpc):=\{h\in\hpc: h(x_{i})=y_{i}, \forall i\in[n]\}$. We say $\hpc$ has an infinite \underline{VC-eluder sequence} $\{(x_{1}, y_{1}),(x_{2}, y_{2}),\ldots\}$ \underline{centered at $h$} if it is realizable and labelled by $h$, and for every integer $k\geq1$, $\{x_{n_{k}+1},\ldots,x_{n_{k}+k}\}$ is a shattered set of $V_{n_{k}}(\hpc)$, where $\{n_{k}\}_{k\in\naturalnumber}$ is a sequence of integers defined as $n_{1}=0$, $n_{k}:=\binom{k}{2}$ for all $k>1$.
\end{definition}
Our result is compact and provides a better concept-dependent characterization of universal rates, leveraging simple combinatorial structures rather than complex conditions.
\begin{theorem}  [\textbf{Bayes-dependent agnostic universal rates}]
  \label{thm:main-theorem-Bayes-dependent}
For every class $\hpc$ with $|\hpc|\geq 3$ and every classifier $\bayesoptimal$, let $\mathcal{P}_{\bayesoptimal}$ be the set of all distributions $P$ such that $\bayesoptimal$ is a Bayes-optimal classifier with respect to $P$, then the following hold:
\setlist{nolistsep}
\begin{itemize}
    \item $\hpc$ is agnostically universally learnable under $\mathcal{P}_{\bayesoptimal}$ by ERM with exact rate $e^{-n}$ if and only if $\hpc$ does not have an infinite eluder sequence centered at $\bayesoptimal$.
    \item $\hpc$ is agnostically universally learnable under $\mathcal{P}_{\bayesoptimal}$ by ERM with exact rate $o(n^{-1/2})$ if and only if $\hpc$ has an infinite eluder sequence but does not have an infinite VC-eluder sequence centered at $\bayesoptimal$.
    \item $\hpc$ requires at least arbitrarily slow rates to be agnostically universally learned under $\mathcal{P}_{\bayesoptimal}$ by ERM if and only if $\hpc$ has an infinite VC-eluder sequence centered at $\bayesoptimal$.
\end{itemize}
\end{theorem}

\section{Target-independent rates}
  \label{sec:target-inpdependent-rates}
In this section, we will introduce the relevant results and their proof sketches for the target-independent universal rates. We will prove each bullet of Theorem~\ref{thm:main-theorem-target-independent} separately (Theorems~\ref{thm:target-independent-exponential-exact-rates}, \ref{thm:target-independent-superroot-exact-rates}, \ref{thm:target-independent-arbitrarily-slow-rates}). We point out that for each bullet, to prove the sufficiency, both a lower bound and an upper bound are required since we are proving an exact rate. All detailed proofs in this section are deferred to Appendix \ref{sec:missing-proofs-target-independent}.

\begin{theorem}  [\textbf{Target-independent $e^{-n}$ exact rates}]
  \label{thm:target-independent-exponential-exact-rates}
A concept class $\hpc$ is agnostically universally learnable with exact rate $e^{-n}$ by ERM if and only if $|\hpc|<\infty$. 
\end{theorem}

\begin{proof} [Proof sketch of Theorem \ref{thm:target-independent-exponential-exact-rates}]
To prove the sufficiency, if $|\hpc|<\infty$, we first have an upper bound from the following lemma
\begin{lemma}
  \label{lem:exponential-upperbound}
Any finite concept class $\hpc$ is agnostically universally learnable at rate $e^{-n}$ by ERM.
\end{lemma}
The proof idea of Lemma~\ref{lem:exponential-upperbound} is, when $|\hpc|<\infty$, Condition~\ref{cond:constant-gap-error-rate} holds with a constant lower bound $\epsilon_{0}$. We bound the excess risk of the worst-case ERM $\hat{h}_{n}$ by the probability of $|\empiricalerrorratehn-\trueerrorratehn|\geq\epsilon_{0}$, and then an exponential rate comes from the Hoeffding's inequality. 

Moreover, since the realizable setting is a special case of agnostic setting, we can get a lower bound on the rate from the following known result 
\begin{lemma}  [{\textbf{\citealp[][]{schuurmans1997characterizing}}}]
  \label{lem:exponential-lowerbound}
Given a class $\hpc$, for any learning algorithm $\lalgo$, there exists a realizable distribution $P$ with respect to $\hpc$ such that $\E[\trueerrorratehn] \geq 2^{-(n+2)}$ for infinitely many $n$.
\end{lemma}
To prove the necessity, we assume to the contrary that $|\hpc|=\infty$, then Lemma~\ref{lem:equivalence-infinite-eluder-and-infinite-cardinality} implies that $\hpc$ has an infinite eluder sequence. A contradiction follows from the following lemma
\begin{lemma}
  \label{lem:super-root-lowerbound}
If $\hpc$ has an infinite eluder sequence centered at $h^{*}$, then $\hpc$ is not agnostically universally learnable under $\mathcal{P}_{h^{*}}$ at rate faster than $o(n^{-1/2})$ by ERM.
\end{lemma}
We prove Lemma~\ref{lem:super-root-lowerbound} by designing a distribution supported on an existing infinite eluder sequence centered at the target function $h^{*}$. In order to have $h^{*}$ being the target function, the distribution has to have decreasing probability masses along the eluder sequence. We then show that, given data generated from the constructed distribution, the worst-case ERM yields universal rates no faster than $o(n^{-1/2})$ by applying an anti-concentration inequality on bounding the probability of the event that more incorrect labeled examples are observed.
\end{proof}

\begin{theorem}  [\textbf{Target-independent $o(n^{-1/2})$ exact rates}]
  \label{thm:target-independent-superroot-exact-rates}
A concept class $\hpc$ is agnostically universally learnable with exact rate $o(n^{-1/2})$ by ERM if and only if $|\hpc|=\infty$ and $\vcdim<\infty$. 
\end{theorem}

\begin{theorem}  [\textbf{Target-independent arbitrarily slow rates}]
  \label{thm:target-independent-arbitrarily-slow-rates}
A concept class $\hpc$ requires at least arbitrarily slow rates to be agnostically universally learned by ERM if and only if $\vcdim=\infty$. 
\end{theorem}

\begin{proof} [Proof sketches of Theorems \ref{thm:target-independent-superroot-exact-rates} and \ref{thm:target-independent-arbitrarily-slow-rates}]
We first prove Theorem~\ref{thm:target-independent-superroot-exact-rates}. To prove the sufficiency, assume that $|\hpc|=\infty$ and $\vcdim<\infty$, the upper bound can be derived from the following lemma
\begin{lemma}
  \label{lem:super-root-upperbound}
Any VC class $\hpc$ is agnostically universally learnable at rate $o(n^{-1/2})$ by ERM.
\end{lemma}
For the proof of Lemma~\ref{lem:super-root-upperbound}, we first utilize a refined version of a classic uniform Bernstein inequality (Proposition~\ref{prop:uniform-bernstein}) to bound the excess risk of the worst-case ERM by $O(\sqrt{P_{\mathcal{X}}\{\hat{h}_{n}(x)\neq h^{*}(x)\}/n})$. Since $\hpc$ is totally bounded in $L_{1}(P_{\mathcal{X}})$ pseudo-metric, we have $P_{\mathcal{X}}\{\hat{h}_{n}(x)\neq h^{*}(x)\}$ is also decreasing as $n$ grows (Lemma~\ref{lem:distance-to-any-hstar}). Finally, an $o(n^{-1/2})$ rate follows from classic localization argument.

Moreover, the lower bound is from the previous Lemma~\ref{lem:equivalence-infinite-eluder-and-infinite-cardinality} and Lemma~\ref{lem:super-root-lowerbound}. To show the necessity, we prove by contradiction. If $|\hpc|<\infty$, then Lemma~\ref{lem:exponential-upperbound} yields the contradiction. If $\vcdim=\infty$, Lemma~\ref{lem:equivalence-infinite-vc-eluder-and-infinite-vcdim} implies that $\hpc$ has an infinite VC-eluder sequence, and then a contradiction follows from the following lemma
\begin{lemma}
  \label{lem:arbitrarily-slow-lowerbound}
If $\hpc$ has an infinite VC-eluder sequence centered at $h^{*}$, then $\hpc$ requires at least arbitrarily slow rates to be agnostically universally learned under $\mathcal{P}_{h^{*}}$ by ERM.
\end{lemma}
The proof of Lemma~\ref{lem:arbitrarily-slow-lowerbound} is similar to the proof of Lemma~\ref{lem:super-root-lowerbound}, and the key point is to construct a data distribution supported on an existing infinite VC-eluder sequence centered at the target function $h^{*}$ with decreasing probability masses assigned to the shattered sets along the sequence. Finally, it is straightforward that Theorem~\ref{thm:target-independent-arbitrarily-slow-rates} holds based on Theorem~\ref{thm:target-independent-superroot-exact-rates} and Lemma~\ref{lem:arbitrarily-slow-lowerbound}. 
\end{proof}

\section{Target-dependent rates}
  \label{sec:target-dependent-rates}
In this section, we give a proof sketch for the target-dependent exact universal rates (Theorem \ref{thm:main-theorem-target-dependent}). All the missing proofs in this section are deferred to Appendix~\ref{sec:missing-proofs-target-dependent}. 

Target-dependent universal rates have been studied in the work of \citet{hanneke2024universal} under the realizable setting. Therein, the authors stated the results as ``$h^{*}$ is (not) agnostically universally learnable at some rate $R$" for a target function $h^{*}$, which is indeed equivalent to ours ``$\hpc$ is (not) agnostically universally learnable at rate $R$ under every distribution $P$ centered at $h^{*}$". We leave the formalized definitions (Definition~\ref{def:target-dependent-agnostic-universal-rates} and \ref{def:asymptotic-target-dependent-agnostic-universally-rates}) to Appendix~\ref{sec:extra-definitions} due to the space limitation. In light of this, we will write the related lemmas in this section following either of the two forms. For the realizable setting, such a definition yields a perfect \emph{tetrachotomy} characterized by certain well-defined combinatorial structures, namely eluder sequence, star-eluder sequence and VC-eluder sequence (see Theorem 2, \citet{hanneke2024universal}). However, for the agnostic case, the aforementioned sequences do not support such a compact theory. We start with a simple example which will develop some initial intuition for what makes the agnostic case different.

\subsection{Condition~\ref{cond:constant-gap-error-rate} and infinite eluder sequence}
  \label{subsec:condition1-and-infinice-eluder-sequence}
\begin{example}  [\textbf{No centered infinite eulder sequence but no faster than $o(n^{-1/2})$ rates}]
  \label{ex:no-infinite-eulder-sequence-but-no-faster-than-o(n^(-1/2))-rates}
Let $\mathcal{X}:=\naturalnumber$, let $h^{*}_{1}$ be defined as $h^{*}_{1}(x):=\mathbbm{1}(x=0)$ and let $h^{*}_{2}$ be defined as $h^{*}_{2}(x)=1-\mathbbm{1}(x=0)$. Furthermore, for any integer $i\geq1$, we define $h_{i}:=\mathbbm{1}(x\in\{0,i\})$. Finally, we define the concept class to be $\hpc:=\{h^{*}_{1},h^{*}_{2}\}\cup\{h_{i}\}_{i\geq1}$. We construct a distribution $P$ as follow:
\begin{align*}
    &P_{\mathcal{X}}(x=0) = P_{\mathcal{X}}(x\geq1) = 1/2 ; \\
    &P(y=1|x=0) = P(y=1|x\geq1) ; \\
    &P(y=1|x=i) = 1/2-\epsilon_{i}, \;\; P(y=0|x=i) = 1/2+\epsilon_{i}, \;\;\forall i\geq1 ; \\
    &\{P_{\mathcal{X}}(x=i)\epsilon_{i}\}_{i\geq1} \text{ is a decreasing sequence} .
\end{align*}
An interesting observation is: $\text{er}_{P}(h^{*}_{1}) = \text{er}_{P}(h^{*}_{2}) = 1/2 = \inf_{h\in\hpc}\trueerrorrateh$, and also $\inf_{h\in\hpc}P_{\mathcal{X}}(x:h(x)\neq h^{*}_{1}(x)) = \inf_{h\in\hpc}P_{\mathcal{X}}(x:h(x)\neq h^{*}_{2}(x)) = 0$, which implies that the constructed distribution $P$ is centered at both $h^{*}_{1}$ and $h^{*}_{2}$. However, it is clear that $\hpc$ has an infinite eluder sequence centered at $h^{*}_{1}$, but does not have an infinite eluder sequence centered at $h^{*}_{2}$. In other words, while aiming to learn $h^{*}_{2}$, an ERM algorithm may try to output $h^{*}_{1}$ instead (since $h^{*}_{1}$ is also an optimal function) and thus has universal rates no faster then $o(n^{-1/2})$ as proved in Lemma \ref{lem:super-root-lowerbound}.
\end{example}
The above Example \ref{ex:no-infinite-eulder-sequence-but-no-faster-than-o(n^(-1/2))-rates} implies that ``the existence/nonexistence of an infinite eluder sequence in $\hpc$ centered at $h^{*}$" is not an ``if and only if" characterization to distinguish between $e^{-n}$ and $o(n^{-1/2})$ ERM universal rates. Then, a naturally follow-up question is what is the desired equivalent characterization. Given a target function $h^{*}$, we consider the aforementioned target-specified Condition~\ref{cond:constant-gap-error-rate}, which basically says that there is a ``gap" between the error rate of the target and the best concept in the class. We will show that $\hpc$ is agnostically universally learnable under 
$\mathcal{P}_{h^{*}}$ at exponential rates if and only if Condition~\ref{cond:constant-gap-error-rate} holds for $h^{*}$. Indeed, the sufficiency holds from the following lemma.
\begin{lemma}
  \label{lem:target-specified-exponential-upperbound}
Let $\hpc$ be any UGC concept class and $h^{*}$ be any classifier. If Condition \ref{cond:constant-gap-error-rate} holds for $h^{*}$, then $h^{*}$ is agnostically universally learnable at rate $e^{-n}$ by ERM.
\end{lemma}
Before proceeding to the necessity, we provide deep insights into the relation between Condition \ref{cond:constant-gap-error-rate} and ``no infinite eluder sequence centered at $h^{*}$". \textbf{First, we can conclude that Condition \ref{cond:constant-gap-error-rate} (holds for $h^{*}$) is stronger than ``$\hpc$ has no infinite eluder sequence centered at $h^{*}$".} On one hand, it guarantees that there is no infinite eluder sequence centered at $h^{*}$, since otherwise the distribution we constructed in the proof of Lemma \ref{lem:super-root-lowerbound} would fail this condition. On the other hand, Example \ref{ex:no-infinite-eulder-sequence-but-no-faster-than-o(n^(-1/2))-rates} reveals that they are inequivalent assumptions: there is no infinite eluder sequence centered at $h^{*}_{2}$, but Condition \ref{cond:constant-gap-error-rate} fails for $h^{*}_{2}$, i.e., $\inf_{h\in\hpc:\trueerrorrateh>\text{er}_{P}(h^{*}_{2})}\{\trueerrorrateh-\text{er}_{P}(h^{*}_{2})\} = \inf_{i\geq1}\{\text{er}_{P}(h_{i})-\text{er}_{P}(h^{*}_{1})\} = \inf_{i\geq1}\{2P_{\mathcal{X}}(x=i)\epsilon_{i}\} = 0$. 
\textbf{Moreover, Condition~\ref{cond:constant-gap-error-rate} is weaker than ``$\hpc$ does not have any infinite eluder sequence".} This can be verified easily by finding an infinite class $\hpc$ such that Condition~\ref{cond:constant-gap-error-rate} holds for some target function $h^{*}$. We give out one of such examples as follow.
\begin{example}  [\textbf{Singletons on $\naturalnumber$}]
  \label{ex:singletons-on-natural-numbers}
    Let $\mathcal{X}=\naturalnumber$ and $\hpc_{\text{singleton},\naturalnumber}:=\{h_{t}:=\mathbbm{1}(x=t)|t\in\mathcal{X}\}$ be the class of all singletons on natural numbers. We consider $\hpc=\hpc_{\text{singleton},\naturalnumber}\cup\{h_{\text{all-1's}}\}$ and a target function $h_{\text{all-1's}}$. Since $|\hpc|=|\hpc_{\text{singleton},\naturalnumber}|=\infty$, it must have an infinite eluder sequence (centered at $h_{\text{all-0's}}$) according to Lemma~\ref{lem:equivalence-infinite-eluder-and-infinite-cardinality}. However, it is straightforward that Condition~\ref{cond:constant-gap-error-rate} holds for $h_{\text{all-1's}}$.
\end{example}
This implies that, while $\hpc$ is not universally learnable at exponential rate under every distribution when $|\hpc|=\infty$, but it can be learned exponentially fast under a subclass of distributions $\mathcal{P}_{h^{*}}$ when $h^{*}$ satisfies some good property. Indeed, if we interpret Condition~\ref{cond:constant-gap-error-rate} as a distribution-specific condition, that is, Condition \ref{cond:constant-gap-error-rate} holds for a distribution $P$ if $\inf_{h\in\hpc:\trueerrorrateh>\inf_{h^{\prime}\in\hpc}\text{er}_{P}(h^{\prime})}\{\trueerrorrateh-\inf_{h^{\prime}\in\hpc}\text{er}_{P}(h^{\prime})\}>0$, then \textbf{Condition~\ref{cond:constant-gap-error-rate} (holds for $P$) is equivalent to ``$\hpc$ does not have any infinite eluder sequence centered at any target function of $P$".} Since $P$ can have multiple target functions, Condition~\ref{cond:constant-gap-error-rate} can be stronger than ``no infinite eluder sequence centered at $h^{*}$". Finally, the following two Lemmas formalize the above analysis and will be helpful to the proofs.
\begin{lemma}
  \label{lem:no-infinite-eluder-sequence-implies-what}
If $\hpc$ does not have an infinite eluder sequence centered at $h^{*}$, then for any distribution $P$ centered at $h^{*}$, let $P_{\mathcal{X}}$ be the associated marginal distribution, the following hold:
\begin{itemize}
    \item[(1)] There exists $h\in\hpc$ such that $P_{\mathcal{X}}\{x:h(x)\neq h^{*}(x)\}=0$.
    \item[(2)] $\inf_{h\in\hpc:P_{\mathcal{X}}\{x:h(x)\neq h^{*}(x)\}>0}P_{\mathcal{X}}\{x:h(x)\neq h^{*}(x)\}>0$.
\end{itemize}
\end{lemma}
It is easy to check that in Example \ref{ex:no-infinite-eulder-sequence-but-no-faster-than-o(n^(-1/2))-rates}, the above (2) holds for $h^{*}_{2}$ while Condition \ref{cond:constant-gap-error-rate} fails. Intuitively, one may think that ``no infinite eluder sequence", as a combinatorial assumption, is insufficient for a guarantee on the joint distribution $P$, but only possible for the marginal distribution $P_{\mathcal{X}}$. This is never a problem for the realizable case since the target is always unique therein.

Another interesting observation is that the (2) in Lemma \ref{lem:no-infinite-eluder-sequence-implies-what} fails for $h^{*}_{1}$, which is the bad target that fails exponential learning rates. Note that if the failure of Condition \ref{cond:constant-gap-error-rate} guarantees such a bad target, then the necessity follows. Indeed, we can consider Condition~\ref{cond:constant-gap-error-rate} as an aforementioned distribution-specific condition and then show that if it fails for some distribution $P$, then there must exist some target $h^{*}$ (with respect to $P$) such that the (2) in Lemma \ref{lem:no-infinite-eluder-sequence-implies-what} also fails.
\begin{lemma}
  \label{lem:condition-constant-gap-error-rate-implies-what}
Let $\hpc$ be any concept class and $P$ be any distribution such that $\hpc$ is totally bounded in $L_{1}(P_{\mathcal{X}})$ pseudo-metric. If Condition \ref{cond:constant-gap-error-rate} fails for $P$, i.e. the following holds:
\begin{equation*}
    \inf_{h\in\hpc:\trueerrorrateh>\inf_{h^{\prime}\in\hpc}\text{er}_{P}(h^{\prime})}\left\{\trueerrorrateh-\inf_{h^{\prime}\in\hpc}\text{er}_{P}(h^{\prime})\right\} = 0 ,
\end{equation*}
then there exists $h^{*}$ such that $\trueerrorratehstar=\inf_{h^{\prime}\in\hpc}\text{er}_{P}(h^{\prime})$ and 
\begin{equation*}
    \inf_{h\in\hpc:P_{\mathcal{X}}\{x:h(x)\neq h^{*}(x)\}>0}P_{\mathcal{X}}\{x:h(x)\neq h^{*}(x)\} = 0 .
\end{equation*}
\end{lemma}
In the proof of Lemma \ref{lem:no-infinite-eluder-sequence-implies-what}, we actually show that if the (2) therein fails for some optimal function $h^{*}$ under distribution $P$, then there exists an infinite eluder sequence centered at that $h^{*}$. Moreover, Lemma \ref{lem:super-root-lowerbound} yields an $o(n^{-1/2})$ lower bound on a centered infinite eluder sequence. Hence, Lemma \ref{lem:condition-constant-gap-error-rate-implies-what} tells us that Condition \ref{cond:constant-gap-error-rate} is not only a sufficient condition to a target-dependent $e^{-n}$ upper bound as shown in Lemma \ref{lem:target-specified-exponential-upperbound}, but also necessary. Altogether, we have the following theorem.
\begin{theorem}  [\textbf{Target-dependent $e^{-n}$ exact rates}]
  \label{thm:target-dependent-exponential-exact-rates}
Let $\hpc$ be any UGC concept class and $h^{*}$ be any classifier. $h^{*}$ is agnostically universally learnable at exact rate $e^{-n}$ by ERM if and only if Condition \ref{cond:constant-gap-error-rate} holds for $h^{*}$. 
\end{theorem}

\begin{proof} [Proof of Theorem \ref{thm:target-dependent-exponential-exact-rates}]
To prove the sufficiency, if Condition \ref{cond:constant-gap-error-rate} holds for $h^{*}$, we know that $h^{*}$ is agnostically universally learnable at rate $e^{-n}$ by ERM according to Lemma \ref{lem:target-specified-exponential-upperbound}, and a lower bound $e^{-n}$ by ERM follows simply from Lemma \ref{lem:exponential-lowerbound}. Hence, the sufficiency holds with exponential exact rates. To prove the necessity, we assume to the contrary that Condition \ref{cond:constant-gap-error-rate} fails for $h^{*}$. By Lemma \ref{lem:condition-constant-gap-error-rate-implies-what} and then Lemma \ref{lem:no-infinite-eluder-sequence-implies-what}, we know that there exists an infinite eluder sequence centered at $h^{*}$. Then Lemma \ref{lem:super-root-lowerbound} yields that $h^{*}$ is not agnostically universally learnable at rate faster than $o(n^{-1/2})$ by ERM. This leads to a contradiction and completes the proof of the necessity.
\end{proof}

\subsection{Condition~\ref{cond:finite-vcd-for-sufficiently-small-ball} and infinite VC-eluder sequence}
  \label{subsec:condition2-and-infinice-vc-eluder-sequence}
Next, we will discuss the target-dependent $o(n^{-1/2})$ exact rate. Recall that `` $\hpc$ has no infinite eluder sequence centered at $h^{*}$" is weaker than Condition \ref{cond:constant-gap-error-rate} (holds for $h^{*}$), which is a fundamental reason for why ``no infinite eluder sequence" is not a correct characterization for the target-dependent exponential rates. As an analogue of Example~\ref{ex:no-infinite-eulder-sequence-but-no-faster-than-o(n^(-1/2))-rates}, it is not hard to construct an example of $(\hpc,P)$, where $P$ has two target functions such that $\hpc$ has no infinite VC-eluder sequence centered at one but has an infinite VC-eluder sequence centered at the other. Such an example indicates that \textbf{``$\hpc$ has no infinite VC-eluder sequence centered at $h^{*}$" is weaker than Condition~\ref{cond:finite-vcd-for-sufficiently-small-ball} (holds for $h^{*}$)}, and is not the correct characterization for the target-dependent super-root rate. Instead, we find out that $\hpc$ is agnostically universally learnable under $\mathcal{P}_{h^{*}}$ at $o(n^{-1/2})$ rate by ERM if and only if the target-specified Condition~\ref{cond:finite-vcd-for-sufficiently-small-ball} holds for $h^{*}$. The following lemma states that Condition \ref{cond:finite-vcd-for-sufficiently-small-ball} stands sufficiently and necessarily for a super-root upper bound.
\begin{lemma}
  \label{lem:target-dependent-superroot-upperbound-iff}
Let $\hpc$ be any UGC concept class and $h^{*}$ be any classifier. Then $h^{*}$ is agnostically universally learnable at rate $o(n^{-1/2})$ by ERM if and only if Condition \ref{cond:finite-vcd-for-sufficiently-small-ball} holds for $h^{*}$. 
\end{lemma} 
Given a target function $h^{*}$ and any distribution $P$ centered at $h^{*}$, let us define the ``$\epsilon$-ball" (of $\hpc$ centered at $h^{*}$) as $\hpc(\epsilon; P, h^{*}) := \{h\in\hpc: 0<\text{er}_{P}(h)-\trueerrorratehstar\leq\epsilon\}$. Condition~\ref{cond:finite-vcd-for-sufficiently-small-ball} basically says that for any distribution $P$ centered at $h^{*}$, the $\epsilon$-ball $\hpc(\epsilon; P, h^{*})$ has finite VC dimension for sufficiently small radius $\epsilon(P)$. However, ``no infinite VC-eluder sequence centered at $h^{*}$" merely provides marginal information, i.e., for any distribution $P$ centered at $h^{*}$, there exists $\epsilon:=\epsilon(P)>0$ such that $\hpc(\epsilon; P_{\mathcal{X}},h^{*}):=\{h\in\hpc:0<P_{\mathcal{X}}\{x:h(x)\neq h^{*}(x)\}\leq\epsilon\}$ has finite VC dimension. To see an analogy, we can interpret Condition \ref{cond:constant-gap-error-rate} as that the $\epsilon$-ball $\hpc(\epsilon; P, h^{*})$ is empty for sufficiently small $\epsilon$. However, ``no infinite eluder sequence centered at $h^{*}$" only implies that for any distribution $P$ centered at $h^{*}$, the marginal $\epsilon$-ball $\hpc(\epsilon; P_{\mathcal{X}}, h^{*})$ is empty for sufficiently small $\epsilon$.

Moreover, it holds similarly that \textbf{Condition~\ref{cond:finite-vcd-for-sufficiently-small-ball} is weaker than ``$\hpc$ does not have any infinite VC-eluder sequence".} This implies, for a class $\hpc$ with $\vcdim=\infty$, while $\hpc$ requires arbitrarily slow rates to be agnostically universally learned by ERM, it is possible to be learned at $o(n^{-1/2})$ rate under a subclass of distributions $\mathcal{P}_{h^{*}}$ when $h^{*}$ satisfies Condition~\ref{cond:finite-vcd-for-sufficiently-small-ball}. Here is an example.
\begin{example}  [{\textbf{\citealp[][Ex.15]{hanneke2024universal}}}]
  \label{ex:standard-non-VC-class}
    Let $\mathcal{X}:=\bigcup_{k\in\naturalnumber}\mathcal{X}_{k}$ be the disjoint union of finite sets with $|\mathcal{X}_{k}|=k$ and $\hpc:=(\bigcup_{k\geq 1}\hpc_{k})\cup\{h_{\text{all-1's}}\}$, where $\hpc_{k}:=\{\mathbbm{1}_{S}: S\subseteq\mathcal{X}_{k}\}$. We consider a target function $h_{\text{all-1's}}$. Now we have $\vcdim=\infty$ and Condition~\ref{cond:finite-vcd-for-sufficiently-small-ball} holds for $h_{\text{all-1's}}$.
\end{example}
If we consider Condition \ref{cond:finite-vcd-for-sufficiently-small-ball} as a distribution-specific condition, i.e., $\hpc(\epsilon; P, h^{*}):=\hpc(\epsilon; P):=\{h\in\hpc:0<\text{er}_{P}(h)-\inf_{h^{\prime}\in\hpc}\text{er}_{P}(h^{\prime})\leq\epsilon\}$, then for any distribution $P$, \textbf{Condition~\ref{cond:finite-vcd-for-sufficiently-small-ball} (holds for $P$) is equivalent to ``$\hpc$ does not have any infinite VC-eluder sequence centered at any target function of $P$".} Finally, together with Theorem \ref{thm:target-dependent-exponential-exact-rates}, we can get the following results of the target-dependent super-root exact rates and arbitrarily slow rates.
\begin{theorem}  [\textbf{Target-dependent $o(n^{-1/2})$ exact rates}]
  \label{thm:target-dependent-superroot-exact-rates}
Let $\hpc$ be any UGC concept class and $h^{*}$ be any classifier. Then $h^{*}$ is agnostically universally learnable at exact rate $o(n^{-1/2})$ by ERM if and only if Condition \ref{cond:finite-vcd-for-sufficiently-small-ball} holds, but Condition \ref{cond:constant-gap-error-rate} fails for $h^{*}$. 
\end{theorem}

\begin{theorem}  [\textbf{Target-dependent arbitrarily slow rates}]
  \label{thm:target-dependent-arbitrarily-slow-rates}
Let $\hpc$ be any concept class and $h^{*}$ be any classifier. Then $h^{*}$ requires at least arbitrarily slow rates to be agnostically universally learned by ERM if and only if Condition \ref{cond:finite-vcd-for-sufficiently-small-ball} fails for $h^{*}$. 
\end{theorem}

\begin{proof} [Proof of Theorems \ref{thm:target-dependent-superroot-exact-rates} and \ref{thm:target-dependent-arbitrarily-slow-rates}]
To prove Theorem \ref{thm:target-dependent-superroot-exact-rates}, note that given Theorem \ref{thm:target-dependent-exponential-exact-rates}, it suffices to prove the part related to Condition \ref{cond:finite-vcd-for-sufficiently-small-ball}. We first prove the sufficiency. To prove an exact rate, the upper bound follows from Lemma \ref{lem:target-dependent-superroot-upperbound-iff}. For the lower bound, note that if Condition \ref{cond:constant-gap-error-rate} fails for $h^{*}$, then there exists an infinite eluder sequence centered at $h^{*}$. Together with Lemma \ref{lem:super-root-lowerbound}, we establish the lower bound. We next prove the necessity using the method of contradiction. Indeed, in the proof of Lemma \ref{lem:target-dependent-superroot-upperbound-iff}, we show that if Condition \ref{cond:finite-vcd-for-sufficiently-small-ball} fails for $h^{*}$, then there exists an infinite VC-eluder sequence centered at $h^{*}$. Then Lemma \ref{lem:arbitrarily-slow-lowerbound} yields a contradiction and completes the proof of necessity. To prove Theorem \ref{thm:target-dependent-arbitrarily-slow-rates}, note that the sufficiency has already been established in the proof of Theorem \ref{thm:target-dependent-superroot-exact-rates}. Finally, the necessity follows immediately by the method of contradiction. 
\end{proof}

\section{Bayes-dependent rates}
  \label{sec:bayes-dependent-rates}
Recall that in Section \ref{sec:target-dependent-rates}, we find out that the target-dependent agnostic universal rates are not characterized by simple combinatorial measures but two contrived target-specified conditions. This flaw motivates us to think whether there is a better function-specified (instead of target-specified) categorization of all data distributions. In this section, we consider to categorize a distribution based on its Bayes-optimal classifier (Definition \ref{def:bayes-optimal-classifier}) and give a theory of Bayes-dependent exact universal rates (Theorem \ref{thm:main-theorem-Bayes-dependent}). Specifically, for every distribution $P$, we show that the agnostic universal rates for learning $\hpc$ under $P$ are determined by the existence/non-existence of infinite eluder sequence and VC-eluder sequence centered at the Bayes-optimal classifier $\bayesoptimal$ with respect to $P$. While the Bayes-optimal classifier may not be unique for a distribution $P$ (at any point $x$ satisfying $P(Y=1|X=x)=P(Y=0|X=x)=1/2$, $\bayesoptimal$ can output either 0 or 1), we can show that if one of the Bayes-optimal classifiers does (not) have a centered infinite eluder/VC-eluder sequence, so do (not) the others. Therefore, this Bayes-dependent characterization provides a complete concept-dependent characterization of universal rates. All the detailed proofs (if required) in this section are deferred to Appendix \ref{sec:missing-proofs-bayes-dependent}.
\begin{theorem}  [\textbf{Bayes-dependent $e^{-n}$ exact rates}]
  \label{thm:bayes-dependent-exponential-exact-rates}
Let $\hpc$ be any concept class and $h$ be any classifier. For any distribution $P$ such that $h$ is a Bayes-optimal classifier with respect to $P$, then $\hpc$ is agnostically universally learnable with exact rate $e^{-n}$ by ERM under $P$ if and only if $\hpc$ does not have an infinite eluder sequence centered at $h$.
\end{theorem}

\begin{theorem}  [\textbf{Bayes-dependent $o(n^{-1/2})$ exact rates}]
  \label{thm:bayes-dependent-superroot-exact-rates}
Let $\hpc$ be any concept class and $h$ be any classifier. For any distribution $P$ such that $h$ is a Bayes-optimal classifier with respect to $P$, then $\hpc$ is agnostically universally learnable at exact rate $o(n^{-1/2})$ by ERM under $P$ if and only if $\hpc$ does not have an infinite VC-eluder sequence, but has an infinite eluder sequence centered at $h$.
\end{theorem}

\begin{theorem}  [\textbf{Bayes-dependent arbitrarily slow rates}]
  \label{thm:bayes-dependent-arbitrarily-slow-rates}
Let $\hpc$ be any concept class and $h$ be any classifier. For any distribution $P$ such that $h$ is a Bayes-optimal classifier with respect to $P$, then $\hpc$ requires at least arbitrarily slow rates to be agnostically universally learned by ERM under $P$ if and only if $\hpc$ has an infinite VC-eluder sequence centered at $h$.
\end{theorem}

\section{Conclusion and Future Work}
  \label{sec:conclusion-and-future-work}
In this work, we have revealed all the possible universal rates of ERM for binary classification under the agnostic setting and provided a complete characterization based on simple complexity notions. Our theory can be shaped into a compact \emph{trichotomy}: there are three possible universal rates of ERM for agnostic learning, being either $e^{-n}$ for finite concept classes, or $o(n^{-1/2})$ for VC classes, or arbitrarily slow for non-VC classes. 

Since the (worst-case) ERM algorithm is a special proper learner, a natural follow-up direction is to study the optimal universal rates of proper learning. Let us consider a simple example of the class of singletons on natural numbers, that is, $\mathcal{H}_{\text{singleton},\naturalnumber}:=\{h_{t}: t\in\naturalnumber\}$, where $h_{t}(x):=\mathbbm{1}\{x=t\}$ for all $x\in\naturalnumber$. For the realizable case, \citet{hanneke2024universal} has proved a $\log(n)/n$ universal rate for ERM. However, a simple proper learner that predicts $h_{\text{all-0's}}$ unless it sees a 1-labeled example (and then predict the singleton on that example), will yield an exponential universal rate. For the agnostic case, we are able to design a proper learning algorithm that achieves $e^{-o(n)}$ agnostic universal rate, while an $o(n^{-1/2})$ universal rate for ERM has been proved in this paper. Therefore, this example provides some intuitions that the universal rates (and probably also the characterizations) for proper learning can be quite different from ERM's.

Another interesting future direction is to understand the universal rates of ERM for multiclass learning (with an infinite label space). Recall that the characterization for the ERM uniform rates of multiclass learning is the graph dimension of $\hpc$ \citep{daniely2015multiclass}, which is defined as the VC dimension of the binary class $\Phi:=\{(x,y)\mapsto\mathbbm{1}\{h(x)\neq y\}: h\in\hpc\}$. Let us consider a simple concept class $\hpc:=\{h_{t}: t\in\naturalnumber\}$ defined on the instance space $\mathcal{X}:=\{x_{0},x_{1}\}$, where $h_{t}(x_{0})=0$ and $h_{t}(x_{1})=t$ for every $t\in\naturalnumber$. For this multiclass learning problem, we know that the uniform rates of ERM (for both realizable and agnostic cases) are exponential. However, since $\Phi$ is infinite, applying our theory to $\Phi$ does not yield a correct characterization for the universal rates of ERM for multiclass learning. Therefore, universal rates of ERM for multiclass learning is not a trivial extension of our results.

% Acknowledgments---Will not appear in anonymized version
\acks{We would like to thank the COLT reviewers for their insightful comments and suggestions.}

\bibliography{ag_ul_erm}

\newpage
\appendix
% \crefalias{section}{appendix} % uncomment if you are using cleveref

\section{Extra definitions}
  \label{sec:extra-definitions}

\begin{definition}
  \label{def:target-dependent-agnostic-universal-rates}
    Following the notations in Definition \ref{def:agnostic-universally-rates}, let $h^{*}:\mathcal{X}\rightarrow\{0,1\}$, we say
    \setlist{nolistsep}
    \begin{itemize}
        \item $h^{*}$ is \underline{agnostically universally learnable at rate $R$ by ERM}, if for every distribution $P$ centered at $h^{*}$ (Definition \ref{def:excess-risk}), there exist parameters $C, c > 0$ such that for every ERM algorithm, $\excessrisk \leq CR(cn)$, for all $n\in\naturalnumber$.
        \item $h^{*}$ is \underline{not agnostically universally learnable at rate faster than $R$ by ERM}, if there exists a distribution $P$ centered at $h^{*}$ and parameters $C, c > 0$ such that there is an ERM algorithm satisfying $\excessrisk \geq CR(cn)$, for infinitely many $n\in\naturalnumber$.
        \item $h^{*}$ is \underline{agnostically universally learnable with exact rate $R$ by ERM}, if $h^{*}$ is agnostically universally learnable at rate $R$ by ERM, and is not agnostically universally learnable at rate faster than $R$ by ERM.
    \end{itemize}
\end{definition}

\begin{definition}
  \label{def:asymptotic-target-dependent-agnostic-universally-rates}
    Following the notations in Definition \ref{def:agnostic-universally-rates}, let $h^{*}:\mathcal{X}\rightarrow\{0,1\}$, we say
    \setlist{nolistsep}
    \begin{itemize}
        \item $h^{*}$ is \underline{agnostically universally learnable at rate $o(R)$ by ERM}, if for every distribution $P$ centered at $h^{*}$, there exists a (distribution-dependent) function $T=o(R)$ such that $\excessrisk \leq CT(cn)$, for all $n\in\naturalnumber$, with some distribution-dependent parameters $C, c > 0$.
        \item $h^{*}$ is \underline{not agnostically universally learnable at rate faster than $o(R)$ by ERM}, if for any function $T=o(R)$, there exists a distribution $P$ centered at $h^{*}$ such that $\excessrisk \geq CT(cn)$, for infinitely many $n\in\naturalnumber$, with some distribution-dependent parameters $C, c > 0$.
        \item $h^{*}$ is \underline{agnostically universally learnable with exact rate $o(R)$ by ERM}, if $h^{*}$ is agnostically universally learnable at rate $o(R)$ by ERM, and is not agnostically universally learnable at rate faster than $o(R)$ by ERM.
    \end{itemize}
\end{definition}

\begin{definition}  [\textbf{Region of disagreement}]
  \label{def:region-of-disagreement}
Let $\mathcal{X}$ be an instance space and $\hpc$ be a concept class. We define the \underline{region of disagreement} of $\hpc$ as $\regionofdisagreement:=\{x\in\mathcal{X}:\exists h,g\in\hpc \text{ s.t. } h(x)\neq g(x)\}$.
\end{definition}

\begin{definition}  [\textbf{Star number}]
  \label{def:star-number}
Let $\mathcal{X}$ be an instance space and $\hpc$ be a concept class. Let $h$ be a classifier, the \underline{star number of $h$}, denoted by $\starnumber_{h}(\hpc)$ or $\starnumber_{h}$ for short, is defined to be the largest integer $s$ such that there exist distinct points $x_{1},\ldots,x_{s}\in\mathcal{X}$ and concepts $h_{1},\ldots,h_{s}\in\hpc$ satisfying $\text{DIS}(\{h,h_{i}\})\cap\{x_{1},\ldots,x_{s}\}=\{x_{i}\}$, for every $1\leq i\leq s$. (We say $\{x_{1},\ldots,x_{s}\}$ is a star set of $\hpc$ \underline{centered at $h$}.) If no such largest integer $s$ exists, we define $\starnumber_{h}=\infty$. The \underline{star number of $\hpc$}, denoted by $\starnumber(\hpc)$ or $\starnumber_{\hpc}$, is defined to be the maximum possible cardinality of a star set of $\hpc$, or $\starnumber_{\hpc}=\infty$ if no such maximum exists.
\end{definition}

\begin{definition}  [\textbf{Star-eluder sequence}]
  \label{def:star-eluder-sequence}
Let $\hpc$ be a concept class and $h$ be a classifier. We say that $\hpc$ has an infinite \underline{star-eluder sequence} $\{(x_{1}, y_{1}),(x_{2}, y_{2}),\ldots\}$ \underline{centered at $h$} , if it is realizable and for every integer $k\geq1$, $\{x_{n_{k}+1},\ldots,x_{n_{k}+k}\}$ is a star set of $V_{n_{k}}(\hpc)$ centered at $h$, where $\{n_{k}\}_{k\in\naturalnumber}$ is a sequence of integers defined as $n_{1}=0$, $n_{k}:=\binom{k}{2}$ for all $k>1$.
\end{definition}

\section{Missing proofs from Section \ref{sec:target-inpdependent-rates}}
  \label{sec:missing-proofs-target-independent}

\begin{lemma}  [\textbf{Lemma \ref{lem:exponential-upperbound} restated}]
  \label{lem:exponential-upperbound-restated}
Any finite concept class $\hpc$ is agnostically universally learnable at rate $e^{-n}$ by ERM.
\end{lemma}

\begin{proof} [Proof of Lemma \ref{lem:exponential-upperbound-restated}]
Let $\hpc$ be any finite concept class and $P$ be any distribution, let $h^{*}:=\argmin_{h\in\hpc}\trueerrorrateh$ be a target function. We define the following distribution-dependnet constant
\begin{equation*}
    \epsilon_{0} := \epsilon_{0}(P) := \min_{h\in\hpc: \trueerrorrateh-\trueerrorratehstar>0}\left\{\trueerrorrateh-\trueerrorratehstar\right\} > 0 .
\end{equation*}
Note that when $|\hpc|<\infty$, such $\epsilon_{0}$ always exists. Now we can bound the excess risk by
\begin{align}
  \label{eq:lem-exponential-upperbound-intermediate-step1}
    \excessrisk \stackrel{\text{\eqmakebox[lemma-finite-class-a][c]{}}}{=} &\E\left[\trueerrorratehn - \trueerrorratehstar\right] \nonumber \\
    \stackrel{\text{\eqmakebox[lemma-finite-class-a][c]{}}}{\leq} &\mathbb{P}\left(\trueerrorratehn - \trueerrorratehstar>0\right) \nonumber \\
    \stackrel{\text{\eqmakebox[lemma-finite-class-a][c]{}}}{=} &\mathbb{P}\left(\exists h\in\hpc: \empiricalerrorrateh\leq\empiricalerrorratehstar, \trueerrorrateh - \trueerrorratehstar>0\right) \nonumber \\
    \stackrel{\text{\eqmakebox[lemma-finite-class-a][c]{}}}{=} &\mathbb{P}\left(\exists h\in\hpc: \empiricalerrorrateh\leq\empiricalerrorratehstar, \trueerrorrateh - \trueerrorratehstar\geq\epsilon_{0}\right) \nonumber \\
    \stackrel{\text{\eqmakebox[lemma-finite-class-a][c]{}}}{\leq} &\mathbb{P}\left(\exists h\in\hpc: \empiricalerrorrateh-\trueerrorrateh \leq \empiricalerrorratehstar-\trueerrorratehstar - \epsilon_{0}\right) .
\end{align}
Suppose that the following event holds:
\begin{equation*}
    \big|\empiricalerrorrateh - \trueerrorrateh\big| < \epsilon_{0}/2, \;\;\forall h\in\hpc .
\end{equation*}
Since $h^{*}\in\hpc$, we will have for any $h\neq h^{*}\in\hpc$,
\begin{equation*}
    \left(\empiricalerrorratehstar-\trueerrorratehstar\right) - \left(\empiricalerrorrateh-\trueerrorrateh\right) < \epsilon_{0}/2 - (-\epsilon_{0}/2) = \epsilon_{0} ,
\end{equation*}
and then the probability in \eqref{eq:lem-exponential-upperbound-intermediate-step1} is zero. Therefore, we can continue the inequality \eqref{eq:lem-exponential-upperbound-intermediate-step1} by
\begin{align*}
    &\mathbb{P}\left(\exists h\in\hpc: \empiricalerrorrateh-\trueerrorrateh \leq \empiricalerrorratehstar-\trueerrorratehstar - \epsilon_{0}\right) \\
    \stackrel{\text{\eqmakebox[lemma-finite-class-b][c]{}}}{\leq} &\mathbb{P}\left(\exists h\in\hpc: \big|\empiricalerrorrateh - \trueerrorrateh\big| \geq \epsilon_{0}/2 \right) \nonumber \\
    \stackrel{\text{\eqmakebox[lemma-finite-class-b][c]{\text{\tiny Union bound}}}}{\leq} &\sum_{h\in\hpc}\mathbb{P}\left(\big|\empiricalerrorrateh - \trueerrorrateh\big| \geq \epsilon_{0}/2\right) \\
    \stackrel{\text{\eqmakebox[lemma-finite-class-b][c]{\tiny \text{Lemma }\ref{lem:hoeffding-ineq}}}}{\leq} &|\hpc|\cdot2e^{-n\epsilon_{0}^{2}/2} ,
\end{align*}
which completes the proof.
\end{proof}

\begin{lemma}  [\textbf{Lemma \ref{lem:super-root-lowerbound} restated}]
  \label{lem:super-root-lowerbound-restated}
If $\hpc$ has an infinite eluder sequence centered at $h^{*}$, then $h^{*}$ is not agnostically universally learnable at rate faster than $o(n^{-1/2})$ by ERM.
\end{lemma}

\begin{proof}[Proof of Lemma \ref{lem:super-root-lowerbound-restated}]
Let $\{(x_{1},y_{1}),(x_{2},y_{2}),\ldots\}$ be an infinite eluder sequence centered at $h^{*}$, i.e. $h^{*}(x_{i})=y_{i}$ for all $i\in\naturalnumber$. We consider the following data distribution $P$ on $\mathcal{X}\times\{0,1\}$:
\begin{equation*}
    P_{\mathcal{X}}(x_{i})=p_{i}, \;\; P\left(y=y_{i}|x_{i}\right)=1/2+\epsilon_{i},\;\; P\left(y=1-y_{i}|x_{i}\right)=1/2-\epsilon_{i},\;\; \forall i\in\naturalnumber ,
\end{equation*}
where $\{p_{i}\}_{i\in\naturalnumber}$ is a sequence of probabilities satisfying $\sum_{i\in\naturalnumber}p_{i}\leq1$ and $\{\epsilon_{i}\}_{i\in\naturalnumber}$ is a sequence of positive real values. Further requirements on these two sequences will be specified later. Now according to the definition of an eluder sequence, for any $i\in\naturalnumber$, there is $h_{i}\in\hpc$ such that $h_{i}(x_{j})=y_{j}$ for all $1\leq j<i$ and $h_{i}(x_{i})=1-y_{i}$. Suppose that
\begin{equation}
  \label{eq:lem-super-root-lowerbound-sequence-assumption1}
    \sum_{j=i+1}^{\infty}p_{j}\epsilon_{j} \leq p_{i}\epsilon_{i},\;\; \forall i\in\naturalnumber .
\end{equation}
Then we have
\begin{align*}
    \text{er}_{P}(h_{i+1}) \leq \sum_{j=1}^{i}p_{j}\left(\frac{1}{2}-\epsilon_{j}\right)+\sum_{j=i+1}^{\infty}p_{j}\left(\frac{1}{2}+\epsilon_{j}\right) \stackrel{\eqref{eq:lem-super-root-lowerbound-sequence-assumption1}}{\leq} \sum_{j\neq i}p_{j}\left(\frac{1}{2}-\epsilon_{j}\right)+p_{i}\left(\frac{1}{2}+\epsilon_{i}\right) \leq \text{er}_{P}(h_{i}) .
\end{align*}
It follows that $\inf_{h\in\hpc}\trueerrorrateh=\inf_{i\in\naturalnumber}\text{er}_{P}(h_{i})=\trueerrorratehstar$. Furthermore, by choosing $\{p_{i}\}_{i\in\naturalnumber}$ to be a decreasing sequence, we also have $\inf_{h\in\hpc}P_{\mathcal{X}}(x:h(x)\neq h^{*}(x))=\inf_{i\in\naturalnumber}P_{\mathcal{X}}(x:h_{i}(x)\neq h^{*}(x))=\inf_{i\in\naturalnumber}p_{i}=0$, that is, $P$ is centered at $h^{*}$. Now let $\{k_{i}\}_{i\in\naturalnumber}$ be an increasing sequence of integers that will be specified later, we consider the event $\mathcal{E}=\mathcal{E}_{1}\cap\mathcal{E}_{2}$, where
\begin{itemize}
    \item[$\mathcal{E}_{1}$:] the number of copies of $(x_{k_{i}},y_{k_{i}})$ is less than the number of copies of $(x_{k_{i}},1-y_{k_{i}})$ in $S_{n}$,
    \item[$\mathcal{E}_{2}$:] $S_{n}$ does not contain any copy of the points in $\{x_{k}:k>k_{i},k\in\naturalnumber\}$.
\end{itemize}
Under $\mathcal{E}$, any ERM algorithm will output a prediction $\lalgo$ labeling $1-y_{k_{i}}$ on $x_{k_{i}}$, or even worse, labeling $1-y_{k}$ on $x_{k}$ for some $k<k_{i}$ if the number of copies of $(x_{k},y_{k})$ in $S_{n}$ is less than the number of copies of $(x_{k},1-y_{k})$. In both cases, it has an excess risk $\trueerrorratehn-\trueerrorratehstar\geq 2p_{k_{i}}\epsilon_{k_{i}}$. First, the probability of $\mathcal{E}_{2}$ is easy to calculate:
\begin{equation*}
    \mathbb{P}(\mathcal{E}_{2}) = \mathbb{P}\left(S_{n}\cap\{x_{k}:k>k_{i}\}=\emptyset\right) = \left(1-\sum_{k>k_{i}}p_{k}\right)^{n} .
\end{equation*}
Next, let $n_{x_{k_{i}}}$ denote the number of copies of $x_{k_{i}}$ in $S_{n}$. Given $n_{x_{k_{i}}}$ fixed, the probability of event $\mathcal{E}_{1}$ can be lower bounded using standard anti-concentration inequality:
\begin{equation}
  \label{eq:lem-super-root-lowerbound-intermediate-step1}
    \mathbb{P}\left(\sum_{k\in[n_{x_{k_{i}}}]}\mathbbm{1}\{Y_{k}=1-y_{k_{i}}\}\geq\frac{n_{x_{k_{i}}}}{2}\right) \stackrel{\text{Lemma }\ref{lem:binomial-bound}}{\geq} \frac{1}{2}\left(1-\sqrt{1-\exp{(-4n_{x_{k_{i}}}\epsilon_{k_{i}}^{2}/(1-4\epsilon_{k_{i}}^{2}))}}\right) . 
\end{equation}
We further assume $n_{x_{k_{i}}} \leq 1/8\epsilon_{k_{i}}^{2}$, which implies that $-4n_{x_{k_{i}}}\epsilon_{k_{i}}^{2}/(1-4\epsilon_{k_{i}}^{2})\geq -8n_{x_{k_{i}}}\epsilon_{k_{i}}^{2}\geq-1$. The RHS of \eqref{eq:lem-super-root-lowerbound-intermediate-step1} can be further extended by
\begin{equation*}
     \frac{1}{2}\left(1-\sqrt{1-\exp{(-4n_{x_{k_{i}}}\epsilon_{k_{i}}^{2}/(1-4\epsilon_{k_{i}}^{2}))}}\right) \geq \frac{1}{2}\left(1-\sqrt{1-e^{-1}}\right) \geq \frac{1}{10} .
\end{equation*}
Let $n:=n_{i}=1/8\epsilon_{i}^{2}$, and so that $n_{x_{k_{i}}}\leq n_{i}=1/8\epsilon_{i}^{2}$ holds with probability 1. We have
\begin{equation*}
    P(\mathcal{E}) = \mathbb{P}(\mathcal{E}_{1}|\mathcal{E}_{2})\mathbb{P}(\mathcal{E}_{2}) \geq \frac{1}{10}\cdot\left(1-\sum_{k>k_{i}}p_{k}\right)^{n_{i}} .
\end{equation*}
Hence we have for any $i\in\naturalnumber$,
\begin{equation*}
    \mathcal{E}(n_{i},P) \geq 2p_{k_{i}}\epsilon_{k_{i}}\cdot\mathbb{P}(\mathcal{E}) \geq \frac{p_{k_{i}}}{5\sqrt{8n_{i}}}\cdot\left(1-\sum_{k>k_{i}}p_{k}\right)^{n_{i}} .
\end{equation*}
To show an $o(1/\sqrt{n})$ lower bound, it suffices to show that for any function $R(n)\rightarrow0$, there exists a sequence of probabilities $\{p_{i}\}_{i\in\naturalnumber}$ and two increasing sequences of integers $\{n_{i}\}_{i\in\naturalnumber}$ and $\{k_{i}\}_{i\in\naturalnumber}$ such that the following hold with some universal constant $C>0$: 
\begin{equation*}
    \sum_{i\in\naturalnumber}p_{i}\leq1,\;\;\; \sum_{k>k_{i}}p_{k}\leq\frac{1}{n_{i}},\;\;\; \sum_{j>i}\frac{p_{j}}{\sqrt{n_{j}}} \leq \frac{p_{i}}{\sqrt{n_{i}}},\;\;\; p_{k_{i}}=CR(n_{i}),\;\;\; \forall i\in\naturalnumber .
\end{equation*}
Valid sequences are constructed in Lemma~\ref{lem:infinite-sequence-design2} and Lemma~\ref{lem:infinite-sequence-design1}. Altogether, this implies that $\excessrisk\gtrsim R(n)\cdot n^{-1/2}$ holds for infinitely many $n\in\naturalnumber$.
\end{proof}

\begin{lemma}  [\textbf{Lemma \ref{lem:super-root-upperbound} restated}]
  \label{lem:super-root-upperbound-restated}
Any VC class $\hpc$ is agnostically universally learnable at rate $o(n^{-1/2})$ by ERM.
\end{lemma}

We first prove a refined version of a classic \emph{uniform Bernstein} inequality, which refines a result of \citet{vapnik1974theory} by a logarithmic factor. The result follows from known localized entropy integral arguments.

\begin{proposition}
  \label{prop:uniform-bernstein}
Let $\hpc$ be any measurable class of functions from $\mathcal{X}$ to $\{0,1\}$ with $\vcdim<\infty$, and $P$ be any probability measure on $\mathcal{X}\times\{0,1\}$. Fix any $n \in\naturalnumber$ and $\delta\in(0,1)$. For any $h,g \in\hpc$, define $\sigma^{2}(h,g)=P_{\mathcal{X}}\{x: h(x)\neq g(x)\}$. Let $S_{n}:=\{(X_{i},Y_{i})\}_{i=1}^{n}\sim P^{n}$, then with probability at least $1-\delta$, every $h,g \in\hpc$ satisfy the following
\begin{align*}
    &\big|\left(\text{er}_{P}(h)-\text{er}_{P}(g)\right) - \left( \hat{\text{er}}_{S_{n}}(h)-\hat{\text{er}}_{S_{n}}(g)\right)\big| \\ 
    \leq &\sqrt{\sigma^{2}(h,g)\frac{c_0}{n}\!\left(\vcdim\log\!\left(\frac{1}{\sigma^{2}(h,g)}\land\frac{n}{\vcdim}\right)+\log\!\left(\frac{1}{\delta}\right)\right)} \\ 
    &{\hskip 4mm}+\frac{c_0}{n}\!\left(\vcdim\log\!\left(\frac{1}{\sigma^{2}(h,g)}\land\frac{n}{\vcdim}\right)+\log\!\left(\frac{1}{\delta}\right)\right)\! ,
\end{align*}
where $c_0$ is a universal constant.
\end{proposition}

\begin{proof}[Proof of Proposition \ref{prop:uniform-bernstein}]
The proof applies a combination of localization \citep{bartlett04,bartlett05,koltchinskii06}, a concentration inequality accounting for variances of loss differences \citep{bousquet2002bennett}, and an entropy integral bound on the rate of uniform convergence which also accounts for variances of loss differences \citep{van-der-Vaart96,gine06,van-der-Vaart11}, together with well-known bounds on the covering numbers of VC classes \citep{haussler95}.

Since the inequality in the proposition is vacuous when $n<2 \vcdim$ (for $c_{0}\geq2$), to focus on non-trivial case we assume that $n\geq2\vcdim$. For any $h,g: \mathcal{X}\rightarrow\{0,1\}$, let $\ell_{h,g}(x,y):=\mathbbm{1}\{h(x)\neq y\}-\mathbbm{1}\{g(x)\neq y\}$, and define $\mathcal{F}:=\{\ell_{h,g}: h,g \in\hpc\}$. Note that $\mathcal{F}$ is a measurable class of functions from $\mathcal{X}\times\{0,1\}$ to $\{-1,0,1\}$. Let $(X,Y)\sim P$ be independent of $S_{n}$. For each $\Delta > 0$, we will apply Lemma \ref{lem:Bennet-concentration} on the following class of zero-mean functions:
\begin{equation*} 
    \mathcal{F}_{\Delta} := \left\{(x,y) \mapsto (f(x,y)-\E[f(X,Y)])/2:\; f \in\mathcal{F},\; \E[(f(X,Y))^{2}]\leq 4\Delta\right\} ,
\end{equation*}
where the $1/2$ constant factor is included to ensure that $\sup_{f\in\mathcal{F}_{\Delta}}\lVert f\rVert_{\infty}\leq1$. We remark that, while Lemma \ref{lem:Bennet-concentration} requires $\mathcal{F}_{\Delta}$ to be countable, the purpose of this assumption in its original proof as stated in \citet{bousquet2002bennett} is merely to ensure the measurability of the suprema $\sup_{f\in\mathcal{F}_{\Delta}}|\sum_{i=1}^{n}f(X_{i},Y_{i})|$ and $\sup_{f\in\mathcal{F}_{\Delta}}|\sum_{i\neq k}f(X_{i},Y_{i})|$, for all integers $k\leq n$. In our case, measurability of these suprema follows directly from the fact that $\mathcal{F}$ is a measurable class (and hence $\mathcal{F}_{\Delta}$ is as well), and thus the conclusions of Lemma \ref{lem:Bennet-concentration} hold for this class $\mathcal{F}_{\Delta}$ as well. 

Note that the variance bound is straightforward from the definition of $\mathcal{F}_{\Delta}$, that is
\begin{align}
  \label{eq:prop-uniform-bernstein-intermediate-step1}
    \sup_{f\in\mathcal{F}_{\Delta}}\mathrm{Var}[f(X,Y)] = &\sup_{f\in\mathcal{F}_{\Delta}}\E\left[(f(X,Y))^{2}\right] \nonumber \\
    = &\sup_{f\in\mathcal{F}:\E[(f(X,Y))^{2}]\leq 4\Delta}\frac{1}{4}\left(\E[(f(X,Y))^{2}]-\left(\E[f(X,Y)]\right)^{2}\right) \nonumber \\
    \leq &\sup_{f\in\mathcal{F}:\E[(f(X,Y))^{2}]\leq 4\Delta}\frac{1}{4}\E[(f(X,Y))^{2}] \leq \Delta.
\end{align}
Furthermore, we have to calculate the following quantity
\begin{equation*}
    \nu := n\Delta + 2\E\left[\sup_{f\in\mathcal{F}_{\Delta}}\bigg|\sum_{i=1}^{n}f(X_{i},Y_{i})\bigg|\right] .
\end{equation*}
By the second inequality in Lemma \ref{lem:Bennet-concentration}, we know that for any $\delta\in(0,1)$, on an event $E_{2}(\Delta, n, \delta)$ of probability at least $1-\delta/2$, the following holds 
\begin{equation}
  \label{eq:prop-uniform-bernstein-intermediate-step2}
    \sup_{f\in\mathcal{F}_{\Delta}}\bigg|\sum_{i=1}^{n}f(X_{i},Y_{i})\bigg| \leq \E\left[\sup_{f\in\mathcal{F}_{\Delta}}\bigg|\sum_{i=1}^{n}f(X_{i},Y_{i})\bigg|\right] + \sqrt{2\nu\ln\left(\frac{2}{\delta}\right)} + \frac{1}{3}\ln\left(\frac{2}{\delta}\right).
\end{equation}
Next, we bound the expectation on the RHS of the above \eqref{eq:prop-uniform-bernstein-intermediate-step2}, based on a uniform entropy integral. Letting $d=\vcdim$, we will begin with an argument that the uniform $\epsilon$-covering numbers of $\mathcal{F}_{\Delta}$ are of order $\left(C/\epsilon\right)^{C^{\prime}d}$ for some universal constants $C$ and $C^{\prime}$. Indeed, the classic result of \citet{haussler95} implies that, for any $\epsilon>0$ and distribution $Q$ over $\mathcal{X}\times\{0,1\}$, the $\epsilon$-covering number of $\hpc$ under the $L_{1}(Q)$ pseudo-metric, denoted by $N(\epsilon,\hpc,L_{1}(Q))$, is at most $\left(C_{0}/\epsilon\right)^{d}$, for some universal constant $C_{0}$ \citep[see also][]{van-der-Vaart96}. Moreover, the functions in the $\epsilon$-net of $\hpc$ (under $L_{1}(Q)$) of size $N(\epsilon,\hpc,L_{1}(Q))$ can be taken to be measurable $\{0,1\}$-valued functions. Specifically, for any $h,g\in\hpc$, let $h_{\epsilon}$ and $g_{\epsilon}$ denote the functions within $L_{1}(Q)$ distance $\epsilon$ to $h$ and $g$, respectively, from the above $\epsilon$-net of $\hpc$. Then for any $(x,y)\in\mathcal{X}\times\{0,1\}$, the following holds: 
\begin{align} 
  \label{eq:prop-uniform-bernstein-intermediate-step3}
    \big|\ell_{h_{\epsilon}, g_{\epsilon}}(x,y)-\ell_{h,g}(x,y)\big| \leq &\big|\mathbbm{1}\{h_{\epsilon}(x)\neq y\}-\mathbbm{1}\{h(x)\neq y\}\big| + \big|\mathbbm{1}\{g_{\epsilon}(x)\neq y\}-\mathbbm{1}\{g(x)\neq y\}\big| \nonumber \\ 
    = &\big|h_{\epsilon}(x)-h(x)\big| + \big|g_{\epsilon}(x)-g(x)\big| .
\end{align}
For any distribution $Q^{\prime}$ on $\mathcal{X}\times\{0,1\}$, we consider $Q:=(1/2)Q^{\prime}+(1/2)P$ for the above analysis. We know that both the $L_{1}(Q^{\prime})$ and $L_{1}(P)$ distances between the functions $\ell_{h_{\epsilon}, g_{\epsilon}}$ and $\ell_{h,g}$ are at most $4\epsilon$, since otherwise it will contradict \eqref{eq:prop-uniform-bernstein-intermediate-step3}. Moreover, by Jensen's inequality, we have
\begin{equation*} 
    \big|\E\left[\ell_{h_{\epsilon}, g_{\epsilon}}(X,Y)\right]-\E\left[\ell_{h, g}(X,Y)\right]\big| \leq \E\left[\big|\ell_{h_{\epsilon}, g_{\epsilon}}(X,Y)-\ell_{h,g}(X,Y)\big|\right] \leq 4\epsilon .
\end{equation*}
It follows that the functions $(\ell_{h_{\epsilon},g_{\epsilon}}-\E[\ell_{h_{\epsilon},g_{\epsilon}}(X,Y)])/2$ and $(\ell_{h,g}-\E[\ell_{h,g}(X,Y)])/2$ have an $L_{1}(Q^{\prime})$ distance at most $4\epsilon$. Thus, the set of functions $(\ell_{h^{\prime},g^{\prime}}-\E[\ell_{h^{\prime},g^{\prime}}(X,Y)])/2$ with $h^{\prime}$, $g^{\prime}$ in the $\epsilon$-net of $\hpc$ (under $L_{1}(Q)$) is a $4\epsilon$-net of $\mathcal{F}_{\Delta}$ under $L_{1}(Q^{\prime})$, and hence $N(4\epsilon,\mathcal{F}_{\Delta},L_{1}(Q^{\prime}))\leq N(\epsilon,\hpc,L_1(Q))^{2}\leq(C_{0}/\epsilon)^{2d}$. Since this argument holds for any $Q^{\prime}$ and $\epsilon>0$, we conclude that the uniform $\epsilon$-covering number of $\mathcal{F}_{\Delta}$, that is $\sup_{Q^{\prime}} N(\epsilon,\mathcal{F}_{\Delta},L_{1}(Q^{\prime}))$, is at most $(4C_{0}/\epsilon)^{2d}$. Moreover, since the above functions in the corresponding $\epsilon$-net of $\mathcal{F}_{\Delta}$ are $[-1,1]$-valued, so are the functions in $\mathcal{F}_{\Delta}$, and it follows that the uniform $\epsilon$-covering number of $\mathcal{F}_{\Delta}$ under the $L_{2}$ pseudo-metric satisfies
\begin{equation}
  \label{eq:VC-type-covering-number}
N(\epsilon,\mathcal{F}_{\Delta},L_{2}) := \sup_{Q}N(\epsilon,\mathcal{F}_{\Delta},L_{2}(Q)) \leq \sup_{Q}N(\epsilon^{2}/2,\mathcal{F}_{\Delta},L_{1}(Q)) \leq \left(\frac{8C_{0}}{\epsilon^{2}}\right)^{2d} = \left(\frac{C_{1}}{\epsilon}\right)^{4d} ,
\end{equation}
with $C_{1}:=\sqrt{8C_{0}}$. Next, we combine this with an entropy integral bound on the expectation on the RHS of \eqref{eq:prop-uniform-bernstein-intermediate-step2}. Specifically, we let
\begin{equation}
  \label{eq:uniform-entropy-integral}
    J\left(\sqrt{\Delta},\mathcal{F}_{\Delta},L_{2}\right) = \int_{0}^{\sqrt{\Delta}}\sqrt{1+\log N(\epsilon,\mathcal{F}_{\Delta},L_{2})}\mathrm{d}\epsilon . 
\end{equation}
Applying Lemma \ref{lem:entropy-uniform-bound} with $\mathcal{F}=\mathcal{F}_{\Delta}$, $\delta=\sqrt{\Delta}$ and envelope function $F=1$, we get
\begin{equation}
  \label{eq:prop-uniform-bernstein-intermediate-step4}
    \E\left[\sup_{f\in\mathcal{F}_{\Delta}}\bigg|\sum_{i=1}^{n}f(X_{i},Y_{i}) \bigg|\right] \leq C_{2}\sqrt{n}J\left(\sqrt{\Delta},\mathcal{F}_{\Delta},L_{2}\right)+ \frac{C_{2}}{\Delta} J\left(\sqrt{\Delta},\mathcal{F}_{\Delta},L_{2}\right)^{2}
\end{equation}
for some universal constant $C_{2}$. Plugging in the bound of uniform covering numbers \eqref{eq:VC-type-covering-number} into the uniform entropy integral \eqref{eq:uniform-entropy-integral}, we have 
\begin{align*} 
    J\left(\sqrt{\Delta},\mathcal{F}_{\Delta},L_{2}\right) \leq \int_{0}^{\sqrt{\Delta}} \sqrt{1+4d\log\left(\frac{C_{1}}{\epsilon}\right)}\mathrm{d}\epsilon \leq C_{3}\sqrt{\Delta d \log\left(\frac{1}{\Delta}\right)} 
\end{align*}
for some universal constant $C_{3}$. Plugging back into \eqref{eq:prop-uniform-bernstein-intermediate-step4} yields
\begin{equation*}
    \E\left[\sup_{f\in\mathcal{F}_{\Delta}}\bigg|\sum_{i=1}^{n}f(X_{i},Y_{i}) \bigg|\right] \leq C_{4}\sqrt{n\Delta d\log\left(\frac{1}{\Delta}\right)}+C_{4}d\log\left(\frac{1}{\Delta}\right)
\end{equation*}
with $C_{4}=C_{2}C_{3}^{2}$. Now together with \eqref{eq:prop-uniform-bernstein-intermediate-step2}, we have that on the event $E_{2}(\Delta,n,\delta)$, 
\begin{equation}
  \label{eq:prop-uniform-bernstein-intermediate-step5}
    \sup_{f\in\mathcal{F}_{\Delta}}\bigg|\sum_{i=1}^{n}f(X_{i},Y_{i})\bigg| \leq C_{4}\sqrt{n\Delta d\log\left(\frac{1}{\Delta}\right)}+C_{4}d\log\left(\frac{1}{\Delta}\right)+\sqrt{2\nu\ln\left(\frac{2}{\delta}\right)}+\frac{1}{3} \ln\left(\frac{2}{\delta}\right) ,
\end{equation}
where
\begin{equation}
  \label{eq:nu-bound}
    \nu = n\Delta + 2\E\left[\sup_{f\in\mathcal{F}_{\Delta}}\bigg|\sum_{i=1}^{n}f(X_{i},Y_{i})\bigg|\right] \leq n\Delta+2C_{4}\sqrt{n\Delta d\log\left(\frac{1}{\Delta}\right)}+2C_{4}d\log\left(\frac{1}{\Delta}\right) .
\end{equation}
Plugging \eqref{eq:nu-bound} into \eqref{eq:prop-uniform-bernstein-intermediate-step5} and relaxing to simplify the expression yields that, on the event $E_{2}(\Delta,n,\delta)$, 
\begin{equation}
  \label{eq:FDelta-Uniform-Bernstein-Bound}
    \sup_{f\in\mathcal{F}_{\Delta}}\bigg|\sum_{i=1}^{n}f(X_{i},Y_{i})\bigg| \leq C_{5}\sqrt{n\Delta\left(d\log\left(\frac{1}{\Delta}\right)+\log\left(\frac{1}{\delta}\right)\right)}+C_{5}\left(d\log\left(\frac{1}{\Delta}\right)+\log\left(\frac{1}{\delta}\right)\right)
\end{equation}
for some universal constant $C_{5}$. Note that for any $h,g\in\hpc$, the following hold:
\begin{align*}
    &\E\left[\left(\ell_{h,g}(X,Y)\right)\right] = \text{er}_{P}(h) - \text{er}_{P}(h) , \\
    &\E\left[\left(\ell_{h,g}(X,Y)\right)^{2}\right] = P_{\mathcal{X}}(x:h(x)\neq g(x)) = \sigma^{2}(h,g) .
\end{align*}
Hence, the definition of $\mathcal{F}_{\Delta}$ implies that for any $h,g\in\hpc$, $(\ell_{h,g}-\E[\ell_{h,g}(X,Y)])/2\in\mathcal{F}_{\Delta}$ if and only if $\sigma^{2}(h,g)\leq4\Delta$, and thus
\begin{align*}
    &\sup_{h,g\in\hpc:\sigma^{2}(h,g)\leq4\Delta}\big|\left(\text{er}_{P}(h)-\text{er}_{P}(g)\right)-\left(\hat{\text{er}}_{S_{n}}(h)-\hat{\text{er}}_{S_{n}}(g)\right)\big| \\ 
    = &\sup_{h,g\in\hpc:\sigma^{2}(h,g)\leq4\Delta}\bigg|\frac{1}{n}\sum_{i=1}^{n}\left(\ell_{h,g}(X_{i},Y_{i})-\E\left[\ell_{h,g}(X,Y)\right]\right)\bigg| = \frac{2}{n}\sup_{f\in\mathcal{F}_{\Delta}}\bigg|\sum_{i=1}^{n}f(X_{i},Y_{i})\bigg| .
\end{align*} 
Together with \eqref{eq:FDelta-Uniform-Bernstein-Bound}, this has the immediate implication that, on the event $E_{2}(\Delta,n,\delta)$, 
\begin{align}
  \label{eq:localized-uniform-bernstein}
    &\sup_{h,g\in\hpc:\sigma^{2}(h,g)\leq4\Delta}\big|\left(\text{er}_{P}(h)-\text{er}_{P}(g)\right)-\left(\hat{\text{er}}_{S_{n}}(h)-\hat{\text{er}}_{S_{n}}(g)\right)\big| \nonumber \\ 
    \leq &2C_{5}\sqrt{\frac{\Delta}{n}\left(d\log\left(\frac{1}{\Delta}\right)+\log\left(\frac{1}{\delta}\right)\right)}+\frac{2C_{5}}{n}\left(d\log\left(\frac{1}{\Delta}\right)+\log\left(\frac{1}{\delta}\right)\right) . 
\end{align}
To complete the proof, we apply the above inequality \eqref{eq:localized-uniform-bernstein} with various levels $\Delta$, so that it holds for each $h,g\in\hpc$ with a value of $\Delta\approx\sigma^{2}(h,g)/4$. In more details, fixing any $\delta\in(0,1)$, for any $i\in\{1,\ldots,\lceil\log_{2}(n/d)\rceil-1\}$, we have that on the event $E_{2}(2^{-i-1},n,2^{-i}\delta)$, every pair of $h,g\in\hpc$ with $2^{-i}<\sigma^{2}(h,g)\leq2^{1-i}$ satisfy
\begin{align}
  \label{eq:prop-uniform-bernstein-intermediate-step6}
    &\big|\left(\text{er}_{P}(h)-\text{er}_{P}(g)\right)-\left(\hat{\text{er}}_{S_{n}}(h)-\hat{\text{er}}_{S_{n}}(g)\right)\big| \nonumber \\ 
    \stackrel{\text{\eqmakebox[prop-uniform-bernstein-a][c]{\eqref{eq:localized-uniform-bernstein}}}}{\leq} &2C_{5}\sqrt{\frac{2^{-i-1}}{n}\left(d\log\left(2^{i+1}\right)+\log\left(\frac{2^{i}}{\delta}\right)\right)}+\frac{2C_{5}}{n}\left(d\log\left(2^{i+1}\right)+\log\left(\frac{2^{i}}{\delta}\right) \right) \nonumber \\ 
    \stackrel{\text{\eqmakebox[prop-uniform-bernstein-a][c]{}}}{\leq} &2C_{5}\sqrt{\frac{\sigma^{2}(h,g)}{2n}\left(d\log\left(\frac{4}{\sigma^{2}(h,g)}\right)+\log\left(\frac{2}{\sigma^{2}(h,g)\delta}\right)\right)} \nonumber \\ 
    {\hskip 4mm} &+\frac{2C_{5}}{n}\left(d\log\left(\frac{4}{\sigma^{2}(h,g)}\right)+\log\left(\frac{2}{\sigma^{2}(h,g)\delta}\right)\right) \nonumber \\ 
    \stackrel{\text{\eqmakebox[prop-uniform-bernstein-a][c]{}}}{\leq} &C_{6}\sqrt{\frac{\sigma^{2}(h,g)}{n}\left(d\log\left(\frac{1}{\sigma^{2}(h,g)}\right)+\log\left(\frac{1}{\delta}\right)\right)}+\frac{C_{6}}{n}\left(d\log\left(\frac{1}{\sigma^{2}(h,g)}\right)+\log\left(\frac{1}{\delta}\right)\right)
\end{align}
for some universal constant $C_{6}$. Moreover, when $i=\lceil\log_{2}(n/d)\rceil$, on the event $E_{2}(2^{-i-1},n,2^{-i}\delta)$, \eqref{eq:localized-uniform-bernstein} implies that every pair of $h,g\in\hpc$ with $\sigma^{2}(h,g)\leq 2^{1-i}$ satisfy 
\begin{align}
  \label{eq:prop-uniform-bernstein-intermediate-step7}
    &\big|\left(\text{er}_{P}(h)-\text{er}_{P}(g)\right)-\left(\hat{\text{er}}_{S_{n}}(h)-\hat{\text{er}}_{S_{n}}(g)\right)\big| \nonumber \\ 
    \leq &2C_{5}\sqrt{\frac{2^{-i-1}}{n}\left(d\log\left(2^{i+1}\right)+\log\left(\frac{2^i}{\delta}\right)\right)}+\frac{2C_{5}}{n}\left(d\log\left(2^{i+1}\right)+\log\left(\frac{2^i}{\delta}\right) \right) \nonumber \\ 
    \leq &2C_{5}\sqrt{\frac{d}{2n^{2}}\left(d\log\left(\frac{4n}{d}\right)+\log\left(\frac{2 n}{d\delta}\right)\right)}+\frac{2C_{5}}{n} \left(d\log\left(\frac{4n}{d}\right)+\log\left(\frac{2n}{d\delta}\right) \right) \nonumber \\ 
    \leq &\left(1+\frac{1}{\sqrt{2}}\right)\frac{2C_{5}}{n}\left(d\log\left(\frac{4n}{d}\right)+\log\left(\frac{2n}{d\delta}\right)\right) \leq \frac{C_7}{n}\left(d\log\left(\frac{n}{d}\right)+\log\left(\frac{1}{\delta}\right)\right)
\end{align}
for some universal constant $C_7$. Note that every pair of $h,g\in\hpc$ either satisfy $2^{-i}<\sigma^{2}(h,g)\leq 2^{1-i}$ for some $i\in\{1,\ldots,\lceil\log_{2}(n/d)\rceil-1\}$, and hence also $\sigma^{2}(h,g)>d/n$, or else satisfy $\sigma^{2}(h,g)\leq 2^{1-i}$ for $i=\lceil\log_{2}(n/d)\rceil$. Thus, altogether from \eqref{eq:prop-uniform-bernstein-intermediate-step6} and \eqref{eq:prop-uniform-bernstein-intermediate-step7} we have that, on the following event 
\begin{equation*} 
    E_{2}(n,\delta) := \bigcap_{i=1}^{\lceil\log_{2}(n/d)\rceil}E_{2}(2^{-i-1},n,2^{-i}\delta) ,
\end{equation*}
every pair of $h,g\in\hpc$ satisfy
\begin{align*}
    &\big|\left(\text{er}_{P}(h)-\text{er}_{P}(g)\right)-\left(\hat{\text{er}}_{S_{n}}(h)-\hat{\text{er}}_{S_{n}}(g)\right)\big| \\ 
    \leq &C_{8}\sqrt{\frac{\sigma^{2}(h,g)}{n}\left(d\log\left(\frac{1}{\sigma^{2}(h,g)} \land \frac{n}{d}\right)+\log\left(\frac{1}{\delta}\right)\right)}+\frac{C_{8}}{n}\left(d\log\left(\frac{1}{\sigma^{2}(h,g)} \land \frac{n}{d}\right)+\log\left(\frac{1}{\delta}\right)\right)
\end{align*}
for some universal constant $C_{8}$. Finally, by union bound, the event $E_{2}(n,\delta)$ holds with probability at least $1-\sum_{i=1}^{\lceil\log_{2}(n/d)\rceil}2^{-i}\delta \geq 1-\delta$, which completes the proof.
\end{proof}

Before proceeding to the proof of the super-root rate upper bound for VC classes (Lemma \ref{lem:super-root-upperbound}), we introduce certain useful notations and give a refined version of the lemma. For any concept class $\hpc$ with $\vcdim<\infty$, any $\epsilon\geq0$, and any distribution $P$ on $\mathcal{X}\times\{0,1\}$, we define 
\begin{equation} 
  \label{eq:epsilon-ball-subclass}
    \hpc(\epsilon; P) = \left\{h\in\hpc: \text{er}_{P}(h)-\inf_{h^{\prime}\in \hpc}\text{er}_{P}(h^{\prime})\leq\epsilon\right\}
\end{equation} 
and also
\begin{equation}   
  \label{eq:distance-leval}
    \sigma^{2}_{\epsilon}(\hpc,P) = \sup_{h\in\hpc(\epsilon;P)}\lim_{\tau\rightarrow0}\inf_{h_{\tau}\in\hpc(\tau;P)}P_{\mathcal{X}}(x: h(x)\neq h_{\tau}(x)) .
\end{equation}
The following lemma states that $\sigma^{2}_{\epsilon}(\hpc,P)$ defined in \eqref{eq:distance-leval} converges to $0$ as $\epsilon\rightarrow0$ at some rate which is probably dependent on $\hpc$, $P$ and the rate of decay of $\epsilon$. A specific form of such a convergence rate might be of independent interest, while showing it converges is enough for our purposes.

\begin{lemma}
  \label{lem:distance-to-any-hstar}
For any concept class $\hpc$ and distribution $P$ on $\mathcal{X}\times\{0,1\}$, if $\hpc$ is totally bounded in $L_{1}(P_{\mathcal{X}})$, then we have
\begin{equation*}
    \lim_{\epsilon\rightarrow0}\sigma^{2}_{\epsilon}(\hpc,P) = 0 .
\end{equation*}
\end{lemma}

\begin{proof}[Proof of Lemma \ref{lem:distance-to-any-hstar}]
The proof uses a similar argument as in the proof of Lemma \ref{lem:condition-constant-gap-error-rate-implies-what}. Since $\sigma^{2}_{\epsilon}(\hpc,P)$ is non-decreasing in $\epsilon$, it suffices to show that there exists a sequence $\epsilon_{i}\rightarrow0$ with $\sigma^{2}_{\epsilon_{i}}(\hpc,P)\rightarrow0$. For each $j\in\naturalnumber$, let $h_{j}\in\hpc(2^{-j};P)$ satisfy 
\begin{equation}
  \label{eq:lem-distance-to-any-hstar-intermediate-step1}
    \lim_{\tau\rightarrow0}\inf_{h_{\tau}\in\hpc(\tau;P)}P_{\mathcal{X}}(x:h_{j}(x)\neq h_{\tau}(x)) \geq \sigma^{2}_{2^{-j}}(\hpc,P)/2 .
\end{equation}
Since $\hpc$ is totally bounded in $L_{1}(P_{\mathcal{X}})$, there exists an increasing sequence $\{j_{i}\}_{i\in\naturalnumber}$ such that $\{h_{j_{i}}\}_{i\in\naturalnumber}$ is a Cauchy sequence in the $L_{1}(P_{\mathcal{X}})$ pseudo-metric, that is, 
\begin{equation}
  \label{eq:lem-distance-to-any-hstar-intermediate-step2}
    \lim_{i\rightarrow\infty}\sup_{i^{\prime}>i}P_{\mathcal{X}}(x:h_{j_{i}}(x)\neq h_{j_{i^{\prime}}}(x)) = 0 . 
\end{equation}
Then we have 
\begin{align*}
    \lim_{i\rightarrow\infty}\sigma^{2}_{2^{-j_{i}}}(\hpc,P) \stackrel{\text{\eqmakebox[lemma-distance-to-any-hstar-a][c]{\eqref{eq:lem-distance-to-any-hstar-intermediate-step1}}}}{\leq} &\lim_{i\rightarrow\infty}2\lim_{\tau\rightarrow0}\inf_{h_{\tau}\in\hpc(\tau;P)}P_{\mathcal{X}}(x: h_{j_{i}}(x)\neq h_{\tau}(x)) \\ 
    \stackrel{\text{\eqmakebox[lemma-distance-to-any-hstar-a][c]{}}}{\leq} &2\lim_{i\rightarrow\infty}\sup_{i^{\prime}>i}P_{\mathcal{X}}(x: h_{j_{i}}(x)\neq h_{j_{i^{\prime}}}(x)) \stackrel{\eqref{eq:lem-distance-to-any-hstar-intermediate-step2}}{=} 0 ,
\end{align*}
where the second inequality follows from the fact that, for any $\tau>0$, all sufficiently large $i^{\prime}$ satisfy $h_{j_{i^{\prime}}}\in\hpc(\tau;P)$. Together with the fact that $\sigma^{2}_{\epsilon}(\hpc,P)\geq0$ for all $\epsilon>0$, the lemma follows.
\end{proof}

Next, for any $n\in\naturalnumber$, we define
\begin{equation*}
    \mathcal{B}_{\epsilon}(n;\hpc,P) = \tilde{c}\sqrt{\sigma^{2}_{\epsilon}(\hpc,P)\frac{\vcdim}{n}\log\left(\frac{1}{\sigma^{2}_{\epsilon}(\hpc,P)} \land \frac{n}{\vcdim}\right)} + \tilde{c}\frac{\vcdim}{n}\log\left(\frac{1}{\sigma^{2}_{\epsilon}(\hpc,P)} \land \frac{n}{\vcdim}\right) ,
\end{equation*}
where $\tilde{c}$ is a universal constant. Following classic localization arguments \citep{bartlett04,bartlett05,koltchinskii06}, we define 
\begin{equation}
  \label{eq:error-level}
    \epsilon_{n}(\hpc,P) = \inf\left\{\epsilon:=\epsilon(n)>0: \mathcal{B}_{\epsilon}(n;\hpc,P)\leq 2\epsilon\right\} .
\end{equation}
To prove Lemma \ref{lem:super-root-upperbound-restated}, it suffices to prove the following lemma, which describes a detailed form of the $o(n^{-1/2})$ rate of decay.

\begin{lemma}  [\textbf{Refined Lemma \ref{lem:super-root-upperbound-restated}}]
    \label{lem:super-root-upperbound-refined}
Let $\hpc$ be any measurable concept class with $\vcdim<\infty$. For any distribution $P$, we define
\begin{equation*}
    \phi_{\mathrm{total}}(n;\hpc,P) = \min\left\{c\sqrt{\frac{\vcdim}{n}}, \epsilon_{n}(\hpc,P)\right\}
\end{equation*}
for some universal constant $c$ (described in the proof). Letting $\datasetstats$, we have
\begin{equation*}
    \excessrisk \leq \phi_{\mathrm{total}}(n;\hpc,P).
\end{equation*}
Moreover, as a function of $n$, $\phi_{\mathrm{total}}(n;\hpc,P)=o( n^{-1/2})$.
\end{lemma}

\begin{remark}
  \label{rem:remark-to-lem-super-root-upperbound-refined}
Consider $\phi_{\mathrm{total}}(n;\hpc,P)$ as a function of $n$, we note that $\epsilon_{n}(\hpc,P)$ decays roughly at the same rate as $\mathcal{B}_{\epsilon_{n}}(n;\hpc,P)$, where the dominant term is $\sqrt{\sigma^{2}_{\epsilon_{n}}(\hpc,P)/n}$. According to Lemma \ref{lem:distance-to-any-hstar}, we know that $\sigma^{2}_{\epsilon_{n}}(\hpc,P)$ decays as $n\rightarrow\infty$, which makes $\epsilon_{n}(\hpc,P)=o(n^{-1/2})$.
\end{remark}

\begin{proof}[Proof of Lemma \ref{lem:super-root-upperbound-refined}]
The proof combines Proposition \ref{prop:uniform-bernstein} with a well-known localization argument. However, while traditional localization arguments would aim to bound the excess risk by an approximate fixed point, this will not be necessary for our purposes, and hence the argument here is somewhat simpler than one aiming for sharp localized excess risk bounds \citep[see e.g.,][]{bartlett04,bartlett05,koltchinskii06}. We will again suppose that $n\geq2\vcdim$ to focus on the non-vacuous case. By the classic result of \citet{talagrand94}, for any $\delta\in(0,1)$, on an event $E_{0}(n,\delta)$ of probability at least $1-\delta$, we have
\begin{equation}
  \label{eq:lem-super-root-upperbound-refined-intermediate-step1}
    \sup_{h\in\hpc}\Big|\hat{\text{er}}_{S_{n}}(h)-\trueerrorrateh\Big| \leq C_{0}\sqrt{\frac{1}{n}\left(\vcdim+\log{\left(\frac{1}{\delta}\right)}\right)} ,
\end{equation}
where $C_0$ is some universal constant. In particular, on the event $E_{0}(n,\delta)$, by applying \eqref{eq:lem-super-root-upperbound-refined-intermediate-step1} for both $\lalgo$ and $\inf_{h\in\hpc}\trueerrorrateh$, we further have
\begin{equation}
  \label{eq:sqrt-VC-bound}
    \trueerrorratehn-\inf_{h\in\hpc}\trueerrorrateh \leq 2C_{0}\sqrt{\frac{1}{n}\left(\vcdim+\log{\left(\frac{1}{\delta}\right)}\right)} .
\end{equation}
We will use this fact to derive two implications. The first one is classical, namely, by setting the RHS of \eqref{eq:sqrt-VC-bound} to $\epsilon$ and solving for $\delta$, we can rewrite this same inequality as follow:
\begin{equation*}
    \mathbb{P}\left(\trueerrorratehn-\inf_{h\in\hpc}\trueerrorrateh>\epsilon\right) \leq \min\left\{1, e^{\vcdim-n\epsilon^{2}/4C_{0}^{2}}\right\} .
\end{equation*}
It follows immediately that
\begin{align}
  \label{eq:lem-super-root-upperbound-refined-implication1}
    \excessrisk = &\int_{0}^{\infty}\mathbb{P}\left(\trueerrorratehn-\inf_{h\in\hpc}\trueerrorrateh>\epsilon\right)\mathrm{d}\epsilon \nonumber \\
    \leq &\int_{0}^{\infty} \min\left\{1, e^{\vcdim-n\epsilon^{2}/4C_{0}^{2}}\right\}\mathrm{d}\epsilon \lesssim \sqrt{\frac{\vcdim}{n}} .
\end{align}
For a second implication, we recall the fact that every VC class satisfies the UGC property (Definition \ref{def:universal-glivenko-cantelli}) and thus is totally bounded in $L_{1}(P_{\mathcal{X}})$ \citep[e.g.][]{haussler95}. Let $d\in\naturalnumber$ be such that $\vcdim\leq d$, and let $h^{*}:\mathcal{X}\rightarrow\{0,1\}$ be any target function, i.e. $\trueerrorratehstar=\inf_{h\in\hpc}\trueerrorrateh$. For any measurable $f:\mathcal{X}\times\{0,1\}\rightarrow\R$, we denote by $\hat{P}_{n}(f):=\frac{1}{n}\sum_{i=1}^{n}f(X_{i},Y_{i})$ and $P(f):=\E[f(X,Y)]$. For any $h\in\hpc$ and $(x,y)\in\mathcal{X}\times\{0,1\}$, we define 
\begin{equation*} 
    \ell_{h}(x,y) := \left(\mathbbm{1}(h(x)\neq y)-\mathbbm{1}(h^{*}(x)\neq y)\right) ,
\end{equation*}
and $\mathcal{F}:=\{\ell_{h}: h\in\hpc\}$, which is a measurable class of functions. Note that the class $\{(x,y)\mapsto\mathbbm{1}(h(x)\neq y): h\in\hpc\}$ is a measurable class of functions with VC dimension at most $d$. Fix any integer $n\geq d$, by applying the classic uniform convergence result of \citet{talagrand94} to this class, we obtain that there is a universal constant $c$ such that, for any $\delta\in(0,1)$, there is an event $E_{1}(n,\delta)$ of probability at least $1-\delta/2$, on which the following holds:
\begin{equation}
  \label{eq:lem-super-root-upperbound-refined-intermediate-step2}
    \sup_{h\in\hpc}\Big|\empiricalerrorrateh-\trueerrorrateh\Big| \leq c\sqrt{\frac{1}{n}\left(d+\ln{\left(\frac{1}{\delta}\right)}\right)} .
\end{equation}
Similarly, we define $\epsilon_{n}(\delta):=2c\sqrt{(d+\ln{(1/\delta)})/n}$. Note that for any $h\in\hpc$, 
\begin{equation*}
    P(\ell_{h}) = \E\left[\ell_{h}(X,Y)\right] = \E\left[\mathbbm{1}(h(X)\neq Y)-\mathbbm{1}(h^{*}(X)\neq Y)\right] = \trueerrorrateh-\inf_{h^{\prime}\in\hpc}\text{er}_{P}(h^{\prime}) ,
\end{equation*}
and hence we have from \eqref{eq:epsilon-ball-subclass} that $ \hpc(\epsilon_{n}(\delta);P)=\{h\in\hpc: P(\ell_{h})\leq\epsilon_{n}(\delta)\}$. The remaining part of the proof aims to establish the following two claims: 
\begin{itemize}
    \item[(1)] that $\lalgo\in\hpc(\epsilon_{n}(\delta);P)$ on the event $E_{1}(n,\delta)$.
    \item[(2)] that the differences $|\hat{P}_{n}(\ell_{h})-P(\ell_{h})|$ are uniformly bounded for all $h\in\hpc(\epsilon_{n}(\delta);P)$ by a function of $\epsilon_{n}(\delta)$ and $P$ that is $o(n^{-1/2})$, on a further event.
\end{itemize}
Then together with a tail integration bound, these two claims will establish the second implication. The first claim follows from the classic argument of \citet{vapnik1974theory}, namely, on $E_{1}(n,\delta)$, 
\begin{align*}
    P(\ell_{\lalgo}) \stackrel{\text{\eqmakebox[lem-super-root-upperbound-refined-a][c]{}}}{=} &\trueerrorratehn-\trueerrorratehstar = \trueerrorrateh-\inf_{h\in\hpc}\trueerrorrateh \\ 
    \stackrel{\text{\eqmakebox[lem-super-root-upperbound-refined-a][c]{\eqref{eq:lem-super-root-upperbound-refined-intermediate-step2}}}}{\leq} &\empiricalerrorratehn-\min_{h\in\hpc}\empiricalerrorrateh+\epsilon_{n}(\delta) = \epsilon_{n}(\delta) .
\end{align*}
To show the second claim, we will apply Proposition \ref{prop:uniform-bernstein}. Note that for any $h\in\hpc(\epsilon_{n}(\delta);P)$, any $\tau>0$ and $h_{\tau}\in\hpc(\tau;P)$, we have that on an event $E_{2}(n,\delta)$ of probability at least $1-\delta/2$,
\begin{align*}
    &\big|\hat{P}_{n}(\ell_{h})-P(\ell_{h})\big| = \big|\left(\trueerrorrateh-\trueerrorratehstar\right)-\left(\empiricalerrorrateh-\empiricalerrorratehstar\right)\big| \\
    \stackrel{\text{\eqmakebox[lem-super-root-upperbound-refined-b][c]{}}}{\leq} &\big|\left(\trueerrorrateh-\text{er}_{P}(h_{\tau})\right)-\left(\empiricalerrorrateh-\hat{\text{er}}_{S_{n}}(h_{\tau})\right)\big| + \big|\text{er}_{P}(h_{\tau})-\trueerrorratehstar+\empiricalerrorratehstar-\hat{\text{er}}_{S_{n}}(h_{\tau})\big| \\
    \stackrel{\text{\eqmakebox[lem-super-root-upperbound-refined-b][c]{\text{\tiny Proposition \ref{prop:uniform-bernstein}}}}}{\leq} &\sqrt{\sigma^{2}(h,h_{\tau})\frac{c_{0}}{n}\left(\vcdim\log\left(\frac{1}{\sigma^{2}(h,h_{\tau})}\land\frac{n}{\vcdim}\right)+\log\left(\frac{2}{\delta}\right)\right)} \\
    \stackrel{\text{\eqmakebox[lem-super-root-upperbound-refined-b][c]{}}}{} &+\frac{c_{0}}{n}\left(\vcdim\log\left(\frac{1}{\sigma^{2}(h,h_{\tau})}\land\frac{n}{\vcdim}\right)+\log\left(\frac{2}{\delta}\right)\right) \\
    \stackrel{\text{\eqmakebox[lem-super-root-upperbound-refined-b][c]{}}}{} &+\big|\text{er}_{P}(h_{\tau})-\trueerrorratehstar+\empiricalerrorratehstar-\hat{\text{er}}_{S_{n}}(h_{\tau})\big| .
\end{align*}
Since the above holds for any $\tau>0$ and any $h_{\tau}\in\hpc(\tau;P)$,
\begin{align}
  \label{eq:lem-super-root-upperbound-refined-intermediate-step3}
    &\sup_{h\in\hpc(\epsilon_{n}(\delta);P)}\big|\hat{P}_{n}(\ell_{h})-P(\ell_{h})\big| \nonumber \\
    \stackrel{\text{\eqmakebox[lem-super-root-upperbound-refined-c][c]{}}}{\leq} &\sup_{h\in\hpc(\epsilon_{n}(\delta);P)}\lim_{\tau\rightarrow0}\inf_{h_{\tau}\in\hpc(\tau;P)}\left\{\sqrt{\sigma^{2}(h,h_{\tau})\frac{c_{0}}{n}\left(\vcdim\log\left(\frac{1}{\sigma^{2}(h,h_{\tau})}\land\frac{n}{\vcdim}\right)+\log\left(\frac{2}{\delta}\right)\right)} \right. \nonumber \\
    \stackrel{\text{\eqmakebox[lem-super-root-upperbound-refined-c][c]{}}}{} &+\frac{c_{0}}{n}\left(\vcdim\log\left(\frac{1}{\sigma^{2}(h,h_{\tau})}\land\frac{n}{\vcdim}\right)+\log\left(\frac{2}{\delta}\right)\right) \nonumber \\
    \stackrel{\text{\eqmakebox[lem-super-root-upperbound-refined-c][c]{}}}{} &\left.+\big|\text{er}_{P}(h_{\tau})-\trueerrorratehstar+\empiricalerrorratehstar-\hat{\text{er}}_{S_{n}}(h_{\tau})\big|\right\} \nonumber \\
    \stackrel{\text{\eqmakebox[lem-super-root-upperbound-refined-c][c]{\eqref{eq:distance-leval}}}}{=} &\sqrt{\sigma_{\epsilon_{n}(\delta)}^{2}(\hpc,P)\frac{c_{0}}{n}\left(\vcdim\log\left(\frac{1}{\sigma_{\epsilon_{n}(\delta)}^{2}(\hpc,P)}\land\frac{n}{\vcdim}\right)+\log\left(\frac{2}{\delta}\right)\right)} \nonumber \\
    \stackrel{\text{\eqmakebox[lem-super-root-upperbound-refined-c][c]{}}}{} &+\frac{c_{0}}{n}\left(\vcdim\log\left(\frac{1}{\sigma_{\epsilon_{n}(\delta)}^{2}(\hpc,P)}\land\frac{n}{\vcdim}\right)+\log\left(\frac{2}{\delta}\right)\right) .
\end{align}
To see that the RHS of \eqref{eq:lem-super-root-upperbound-refined-intermediate-step3} is $o(n^{-1/2})$, we note that $\epsilon_{n}(\delta)$ decays as $n\rightarrow\infty$. Furthermore, Lemma \ref{lem:distance-to-any-hstar} states that $\sigma^{2}_{\epsilon}(\hpc,P)\rightarrow0$ as $\epsilon\rightarrow0$. Therefore, $\sigma_{\epsilon_{n}(\delta)}^{2}(\hpc,P)$ decays at some rate as a function of $n$, which makes the RHS of \eqref{eq:lem-super-root-upperbound-refined-intermediate-step3} faster than $n^{-1/2}$. Now we have
\begin{align*}
    \trueerrorratehn-\trueerrorratehstar = &P(\ell_{\lalgo}) \\
    \leq &\big|\hat{P}_{n}(\ell_{\lalgo})-P(\ell_{\lalgo})\big|+\hat{P}_{n}(\ell_{\lalgo}) \\
    = &\big|\hat{P}_{n}(\ell_{\lalgo})-P(\ell_{\lalgo})\big|+\empiricalerrorratehn-\empiricalerrorratehstar \\
    = &\big|\hat{P}_{n}(\ell_{\lalgo})-P(\ell_{\lalgo})\big|+\min_{h\in\hpc(\epsilon_{n}(\delta);P)}\empiricalerrorrateh-\empiricalerrorratehstar \\
    = &\big|\hat{P}_{n}(\ell_{\lalgo})-P(\ell_{\lalgo})\big|+\min_{h\in\hpc(\epsilon_{n}(\delta);P)}\hat{P}_{n}(\ell_{h}) \\
    \leq &\big|\hat{P}_{n}(\ell_{\lalgo})-P(\ell_{\lalgo})\big|+\min_{h\in\hpc(\epsilon_{n}(\delta);P)}\left\{\big|\hat{P}_{n}(\ell_{h})-P(\ell_{h})\big|+P(\ell_{h})\right\} \\
    \leq &2\sqrt{\sigma_{\epsilon_{n}(\delta)}^{2}(\hpc,P)\frac{c_{0}}{n}\left(\vcdim\log\left(\frac{1}{\sigma_{\epsilon_{n}(\delta)}^{2}(\hpc,P)}\land\frac{n}{\vcdim}\right)+\log\left(\frac{2}{\delta}\right)\right)} \\
    &+ \frac{2c_{0}}{n}\left(\vcdim\log\left(\frac{1}{\sigma_{\epsilon_{n}(\delta)}^{2}(\hpc,P)}\land\frac{n}{\vcdim}\right)+\log\left(\frac{2}{\delta}\right)\right) ,
\end{align*}
on the event $E_{1}(n,\delta)\cap E_{2}(n,\delta)$ of probability at least $1-\delta$. Then, a tail integration bound yields
\begin{align}
  \label{eq:lem-super-root-upperbound-refined-implication2}
    \excessrisk = &\int_{0}^{\infty}\mathbb{P}\left(\trueerrorratehn-\trueerrorratehstar>\epsilon\right)\mathrm{d}\epsilon \nonumber \\
    \leq &\frac{\tilde{c}}{2}\sqrt{\sigma_{\epsilon_{n}}^{2}(\hpc,P)\frac{\vcdim}{n}\log\left(\frac{1}{\sigma_{\epsilon_{n}}^{2}(\hpc,P)}\land\frac{n}{\vcdim}\right)} \nonumber \\
    &+\frac{\tilde{c}}{2}\cdot\frac{\vcdim}{n}\log\left(\frac{1}{\sigma_{\epsilon_{n}}^{2}(\hpc,P)}\land\frac{n}{\vcdim}\right) =: \frac{\mathcal{B}_{\epsilon_{n}}(n;\hpc,P)}{2} \leq \epsilon_{n}(\hpc,P) ,
\end{align}
where $\epsilon_{n}=\theta(\sqrt{d/n})$ and $\tilde{c}$ is a universal constant. Altogether, \eqref{eq:lem-super-root-upperbound-refined-implication1} and \eqref{eq:lem-super-root-upperbound-refined-implication2} complete the proof.
\end{proof}

\begin{lemma}  [\textbf{Lemma \ref{lem:arbitrarily-slow-lowerbound} restated}]
  \label{lem:arbitrarily-slow-lowerbound-restated}
If $\hpc$ has an infinite VC-eluder sequence centered at $h^{*}$, then $h^{*}$ requires at least arbitrarily slow rates to be agnostically universally learned by ERM.
\end{lemma}

\begin{proof}[Proof of Lemma \ref{lem:arbitrarily-slow-lowerbound-restated}]
Let $\{(x_{1},y_{1}),(x_{2},y_{2}),\ldots\}$ be an infinite VC-eluder sequence in $\hpc$ centered at $h^{*}$. For notation simplicity, we denote $\mathcal{X}_{k}:=\{x_{n_{k}+1},\ldots,x_{n_{k}+k}\}$ for every integer $k\geq1$, where $\{n_{k}\}_{k\geq1}$ is defined in Definition \ref{def:vc-eluder-sequence}. According to the definition of the VC-eluder sequence, we know that for any $k\geq1$, $\mathcal{X}_{k}$ is a shattered set of $\{h\in\hpc: h(x_{i})=h^{*}(x_{i}), \forall i\leq n_{k}\}$. In other words, for any $k\geq1$ and any label vector $\vy_{k}\in\{0,1\}^{k}$, there exists $h_{k,\vy_{k}}\in\hpc$ such that $h_{k,\vy_{k}}(x_{i})=h^{*}(x_{i})$ for all $i\leq n_{k}$ and $h_{k,\vy_{k}}((x_{n_{k}+1},\ldots,x_{n_{k}+k}))=\vy_{k}$. We then consider the following distribution $P$ defined on the VC-eluder sequence:
\begin{align*}
    &P_{\mathcal{X}}(x\in\mathcal{X}_{k})=p_{k}, \;\; P_{\mathcal{X}}(x)=p_{k}/k, \;\;\; \forall x\in\mathcal{X}_{k} ; \\
    &P(y=h^{*}(x)|x)=1/2+\epsilon_{k}, \;\; P(y=1-h^{*}(x)|x)=1/2-\epsilon_{k}, \;\;\; \forall x\in\mathcal{X}_{k}, \forall k\geq1 ,
\end{align*}
where $\{p_{k}\}_{k\geq1}$ is a sequence of probabilities satisfying $\sum_{k\geq1}p_{k}\leq 1$ and $\{\epsilon_{k}\}_{k\geq1}$ is a sequence of positive reals. We would like the constructed distribution $P$ to be centered at $h^{*}$. To this end, we first note that for any $h\in\hpc$, it holds $\trueerrorrateh\leq\trueerrorratehstar$. Moreover, for any $k\geq1$, we consider a label vector $\vy_{k}\in\{0,1\}^{k}$ with every coordinate does not match $h^{*}$. Following the same argument used in the proof of Lemma \ref{lem:super-root-lowerbound}, suppose that $\sum_{j\geq t+1}p_{k_{j}}\epsilon_{k_{j}}\leq p_{k_{t}}\epsilon_{k_{t}}$ for all $t\geq1$ on some increasing subsequence $\{k_{t}\}_{t\geq1}$, then we have $\text{er}_{P}(h_{k_{t+1},\vy_{k_{t+1}}})\leq\text{er}_{P}(h_{k_{t},\vy_{k_{t}}})$, and thus $\inf_{h\in\hpc}\trueerrorrateh=\lim_{t\rightarrow\infty}\text{er}_{P}(h_{k_{t},\vy_{k_{t}}})=\lim_{k\rightarrow\infty}\text{er}_{P}(h_{k,\vy_{k}})=\trueerrorratehstar$. Furthermore, to have $\inf_{h\in\hpc}P_{\mathcal{X}}(x:h(x)\neq h^{*}(x))=\lim_{k\rightarrow\infty}p_{k}/k=0$, it suffices to assume further $\{p_{k}\}_{k\geq1}$ is decreasing.

Now let $R(n)\rightarrow0$ be any rate function, and we aim to show that $\hpc$ cannot be agnostically universally learned at rate faster than $R(n)$ under the distribution $P$ by ERM. To this end, we let $\datasetstats$. For any $k\geq1$ and any $l\in[k]$, we consider the following event:
\begin{equation*}
    \Lambda_{k,l}(S_{n}) := \left\{ \sum_{i=1}^{n}\mathbbm{1}(x_{i}\in\mathcal{X}_{>k})=0 \text{ and } \sum_{i=1}^{n}\mathbbm{1}(x_{i}=x_{n_{k}+l})=0 \right\} ,
\end{equation*}
that is, the dataset $S_{n}$ contains no copy of the points in $\mathcal{X}_{>k}\cup\{x_{n_{k}+l}\}$. Note that if $\Lambda_{k,l}(S_{n})$ holds, then any ERM algorithm will have an excess risk at least $2p_{k}\epsilon_{k}/k$. Moreover, the probability of such an event is not hard to bound:
\begin{equation*}
    \mathbb{P}(\Lambda_{k,l}(S_{n})) = \mathbb{P}\left(\left\{ x_{i}\in\bigcup_{j\leq k}\mathcal{X}_{j}\setminus\{x_{n_{k}+l}\}, \forall i\in[n] \right\}\right) = \left(1-\sum_{j>k}p_{j}-\frac{p_{k}}{k}\right)^{n} .
\end{equation*}
We simply choose $\epsilon_{k}=1/4$ for all $k\geq1$, and then the next key step is to show there is a sequence of integers $\{k_{t}\}_{t\geq1}$ such that all the following hold with some universal constant $1/2\leq C\leq1$: 
\begin{itemize}
    \item[(1)] $\sum_{k\geq1}p_{k}\leq 1$, \;\; $\{p_{k}\}_{k\geq1} \downarrow$, \;\; $\sum_{j\geq t+1}p_{k_{j}}\epsilon_{k_{j}}\leq p_{k_{t}}\epsilon_{k_{t}}$ for all $t\geq1$;
    \item[(2)] $\sum_{j>k_{t}}p_{j} \leq 1/n_{t}$ for all $t\geq1$;
    \item[(3)] $n_{t}p_{k_{t}} \leq k_{t}$ for all $t\geq1$;
    \item[(4)] $p_{k_{t}} = CR(n_{t})$ for all $t\geq1$.
\end{itemize}
By Lemma \ref{lem:infinite-sequence-design2}, we know that there exists a sequence of probabilities $\{p_{k}\}_{k\geq1}$ satisfying $\sum_{k\geq1}p_{k}=1$ and two sequences of integers $\{k_{t}\}_{t\geq1}$ and $\{n_{t}\}_{t\geq1}$ such that the above (2-4) hold. We consider any rate function $R(n)$ that decays slower than $1/n$ (note that if this holds, then it also holds for even faster $R(n)$ since we are showing a lower bound), and immediately get (1), i.e.
\begin{equation*}
    \sum_{j\geq t+1}p_{k_{j}} \leq \sum_{j\geq k_{t+1}}p_{j} \leq \sum_{j>k_{t}}p_{j} \stackrel{(2)}{\leq} \frac{1}{n_{t}} \leq CR(n_{t}) = p_{k_{t}} .
\end{equation*}
Finally, we have for all $t\geq1$ (that is, for infinitely many $n\in\naturalnumber$),
\begin{equation*}
    \mathcal{E}(n_{t},P) \geq \sum_{l\in[k_{t}]}\frac{2p_{k_{t}}\epsilon_{k_{t}}}{k_{t}}\left(1-\sum_{j>k_{t}}p_{j}-\frac{p_{k_{t}}}{k_{t}}\right)^{n_{t}} \stackrel{(2-4)}{\geq} \frac{C}{2}\left(1-\frac{2}{n_{t}}\right)^{n_{t}}R(n_{t}) \geq \frac{R(n_{t})}{36} .
\end{equation*}
\end{proof}

\section{Missing proofs from Section \ref{sec:target-dependent-rates}}
  \label{sec:missing-proofs-target-dependent}

\begin{lemma}  [\textbf{Lemma \ref{lem:target-specified-exponential-upperbound} restated}]
  \label{lem:target-specified-exponential-upperbound-restated}
Let $\hpc$ be a concept class satisfying the UGC property (see Definition \ref{def:universal-glivenko-cantelli}). If the Condition \ref{cond:constant-gap-error-rate} holds for $h^{*}$, then $h^{*}$ is agnostically universally learnable at rate $e^{-n}$ by ERM.
\end{lemma}

\begin{proof}[Proof of Lemma \ref{lem:target-specified-exponential-upperbound-restated}]
Let P be any distribution centered at $h^{*}$. Since Condition \ref{cond:constant-gap-error-rate} holds, we can define
\begin{equation*}
    \epsilon_{0} := \inf_{h\in\hpc:\trueerrorrateh>\trueerrorratehstar}\left\{\trueerrorrateh-\trueerrorratehstar\right\} > 0 ,
\end{equation*}
which is a distribution-dependent constant. Furthermore, since $\hpc$ is a UGC class, we know that there exists a distribution-dependent integer $n_{0}:=n_{0}(P)<\infty$ such that
\begin{equation}
  \label{eq:lem-target-specified-exponential-upperbound-intermediate-step1}
    \E\left[\sup_{h\in\hpc}\Big|\empiricalerrorrateh-\trueerrorrateh\Big|\right] \leq \epsilon_{0}/4, \;\;\; \forall n\geq n_{0} .
\end{equation}
Assume that $n\geq n_{0}$, then Lemma \ref{lem:distance-empiricalerror-and-trueerror} implies that with probability of at least $1-\delta$,
\begin{equation*}
    \sup_{h\in\hpc}\Big|\empiricalerrorrateh-\trueerrorrateh\Big| \leq \E\left[\sup_{h\in\hpc}\Big|\empiricalerrorrateh-\trueerrorrateh\Big|\right] + \sqrt{\frac{1}{2n}\ln{\left(\frac{4}{\delta}\right)}} \stackrel{\eqref{eq:lem-target-specified-exponential-upperbound-intermediate-step1}}{\leq}\frac{\epsilon_{0}}{4} + \sqrt{\frac{1}{2n}\ln{\left(\frac{4}{\delta}\right)}} .
\end{equation*}
By choosing $\delta:=4e^{-n\epsilon_{0}^{2}/8}$, we have for any $n\geq n_{0}$, with probability of at least $1-4e^{-n\epsilon_{0}^{2}/8}$,
\begin{equation}
  \label{eq:lem-target-specified-exponential-upperbound-intermediate-step2}
    \sup_{h\in\hpc}\Big|\empiricalerrorrateh-\trueerrorrateh\Big| \leq \epsilon_{0}/2 .
\end{equation}
Since $\hpc$ does not have an infinite eluder sequence centered at $h^{*}$ (based on Condition \ref{cond:constant-gap-error-rate}), we know by Lemma \ref{lem:no-infinite-eluder-sequence-implies-what} that there exists $h\in\hpc$ such that $P_{\mathcal{X}}\{x:h(x)\neq h^{*}(x)\}=0$. Now note that if the inequality in \eqref{eq:lem-target-specified-exponential-upperbound-intermediate-step2} holds, we will have
\begin{equation*}
    \trueerrorratehn \leq \empiricalerrorratehn+\epsilon_{0}/2 \leq \empiricalerrorrateh+\epsilon_{0}/2 \leq \trueerrorrateh+\epsilon_{0} = \trueerrorratehstar+\epsilon_{0} .
\end{equation*}
Therefore, it follows that for all $n\geq n_{0}$,
\begin{equation*}
    \excessrisk \leq \mathbb{P}\left(\trueerrorratehn-\trueerrorratehstar\geq\epsilon_{0}\right) \leq \mathbb{P}\left(\sup_{h\in\hpc}\Big|\empiricalerrorrateh-\trueerrorrateh\Big|>\epsilon_{0}/2\right) \leq 4e^{-n\epsilon_{0}^{2}/8}.
\end{equation*}
Finally, we obtain $\excessrisk\leq\max\{4,e^{n_{0}\epsilon_{0}^{2}/8}\}\cdot e^{-n\epsilon_{0}^{2}/8}$ for all $n\in\naturalnumber$.
\end{proof}

\begin{lemma}  [\textbf{Lemma \ref{lem:no-infinite-eluder-sequence-implies-what} restated}]
  \label{lem:no-infinite-eluder-sequence-implies-what-restated}
If $\hpc$ does not have an infinite eluder sequence centered at $h^{*}$, then for any distribution $P$ centered at $h^{*}$, let $P_{\mathcal{X}}$ be the associated marginal distribution, the following hold:
\setlist{nolistsep}
\begin{itemize}
    \item[(1)] There exists $h\in\hpc$ such that $P_{\mathcal{X}}\{x:h(x)\neq h^{*}(x)\}=0$.
    \item[(2)] $\inf_{h\in\hpc:P_{\mathcal{X}}\{x:h(x)\neq h^{*}(x)\}>0}P_{\mathcal{X}}\{x:h(x)\neq h^{*}(x)\}>0$.
\end{itemize}
\end{lemma}

\begin{proof}[Proof of Lemma \ref{lem:no-infinite-eluder-sequence-implies-what-restated}]
Since $P$ is centered at $h^{*}$, $\inf_{h\in\hpc}P_{\mathcal{X}}(x:h(x)\neq h^{*}(x))=0$. Then for every $k\in\naturalnumber$, let $r_{k}:=2^{-k}$ and we know that there exists a concept $h_{k}\in\hpc$ such that $P_{\mathcal{X}}(x:h_{k}(x)\neq h^{*}(x))\leq r_{k}$. Let $\{x_{1},x_{2},\ldots\}\sim P_{\mathcal{X}}^{\infty}$, then for every integer $t\geq1$, a union bound gives
\begin{equation*}
    \sum_{k=1}^{\infty}\mathbb{P}\left(h_{k}(x_{s})\neq h^{*}(x_{s}) \text{ for some } s\leq t\right) \leq t\sum_{k=1}^{\infty}P_{\mathcal{X}}(x:h_{k}(x)\neq h^{*}(x)) \leq t\sum_{k=1}^{\infty}r_{k} < \infty .
\end{equation*}
By the Borel-Cantelli's lemma, we further have
\begin{equation}
  \label{eq:lem-target-specified-exponential-upperbound-intermediate-step5}
    \mathbb{P}\left(\lim_{k\rightarrow\infty}h_{k}(x_{s})=h^{*}(x_{s}), \forall s\leq t\right) = 1 ,
\end{equation}
that is, with probability one, there exists for every $t\geq 1$ a concept $h\in\hpc$ satisfying $h(x_{s})=h^{*}(x_{s})$ for all $s\leq t$. We assume first that for any $h\in\hpc$, $P_{\mathcal{X}}\{x:h(x)\neq h^{*}(x)\}>0$. Since $\hpc$ does not have an infinite eluder sequence centered at $h^{*}$, then there exists a distribution-dependent integer $\tau:=\tau(P_{\mathcal{X}},h^{*})<\infty$ such that for $x_{1},\ldots,x_{\tau}\overset{i.i.d.}{\sim} P_{\mathcal{X}}$,
\begin{equation*}
    \mathbb{P}\left(\{h\in\hpc: h(x_{i})=h^{*}(x_{i}), \forall 1\leq i\leq \tau\}=\emptyset\right) \geq 1/2 ,
\end{equation*}
which contradicts \eqref{eq:lem-target-specified-exponential-upperbound-intermediate-step5}. Hence, there must exist some $h\in\hpc$ such that $P_{\mathcal{X}}\{x:h(x)\neq h^{*}(x)\}=0$.

To prove the second statement, we assume to the contrary that 
\begin{equation*}
    \inf_{h\in\hpc:P_{\mathcal{X}}\{x:h(x)\neq h^{*}(x)\}>0}P_{\mathcal{X}}\{x:h(x)\neq h^{*}(x)\} = 0 ,    
\end{equation*}
and then directly construct an infinite eluder sequence centered at $h^{*}$ via the following procedure:
\setlist{nolistsep}
\begin{itemize}
    \item[1.] Let $\epsilon_{1}>0$, then there exists $h_{1}\in\hpc$ such that $0<P_{\mathcal{X}}\{x:h_{1}(x)\neq h^{*}(x)\}<\epsilon_{1}$. We pick any point $x_{1}\in\text{DIS}(\{h_{1},h^{*}\})$ such that $P_{\mathcal{X}}(x_{1})>0$.
    \item[2.] Now we choose $\epsilon_{2}=P_{\mathcal{X}}(x_{1})>0$, and know there exists $h_{2}\in\hpc$ such that $0<P_{\mathcal{X}}\{x:h_{2}(x)\neq h^{*}(x)\}<\epsilon_{2}$. Note that such $h_{2}$ must satisfy $h_{2}(x_{1})=h^{*}(x_{1})$, since otherwise we will have $P_{\mathcal{X}}\{x:h_{2}(x)\neq h^{*}(x)\}\geq P_{\mathcal{X}}(x_{1})=\epsilon_{2}$. We then pick any point $x_{2}\in\text{DIS}(\{h_{2},h^{*}\})$ such that $P_{\mathcal{X}}(x_{2})>0$.
    \item[3.] Similarly, we set $\epsilon_{3}=P_{\mathcal{X}}(x_{2})>0$ and find $h_{3}\in\hpc$ such that $0<P_{\mathcal{X}}\{x:h_{3}(x)\neq h^{*}(x)\}<\epsilon_{3}$. Such $h_{3}$ must satisfy $h_{3}(x_{1})=h^{*}(x_{1})$ and $h_{3}(x_{2})=h^{*}(x_{2})$, since otherwise we will reach a contradiction that $P_{\mathcal{X}}\{x:h_{3}(x)\neq h^{*}(x)\}\geq\min\{P_{\mathcal{X}}(x_{1}),P_{\mathcal{X}}(x_{2})\}=P_{\mathcal{X}}(x_{2})=\epsilon_{3}$. We then pick any point $x_{3}\in\text{DIS}(\{h_{3},h^{*}\})$ such that $P_{\mathcal{X}}(x_{3})>0$.
    \item[...]
\end{itemize}
Finally, it is clear that $\{(x_{1},h^{*}(x_{1})),(x_{2},h^{*}(x_{2})),\ldots\}$ is an infinite eluder sequence centered at $h^{*}$, witnessed by the concepts $\{h_{1},h_{2},\ldots\}\subseteq\hpc$.
\end{proof}

\begin{lemma}  [\textbf{Lemma \ref{lem:condition-constant-gap-error-rate-implies-what} restated}]
  \label{lem:condition-constant-gap-error-rate-implies-what-restated}
Let $\hpc$ be any concept class and $P$ be any distribution such that $\hpc$ is totally bounded in $L_{1}(P_{\mathcal{X}})$ pseudo-metric. If the Condition \ref{cond:constant-gap-error-rate} fails for $P$, i.e. the following holds:
\begin{equation*}
    \inf_{h\in\hpc:\trueerrorrateh>\inf_{h^{\prime}\in\hpc}\text{er}_{P}(h^{\prime})}\left\{\trueerrorrateh-\inf_{h^{\prime}\in\hpc}\text{er}_{P}(h^{\prime})\right\} = 0 ,
\end{equation*}
then there exists $h^{*}$ such that $\trueerrorratehstar=\inf_{h^{\prime}\in\hpc}\text{er}_{P}(h^{\prime})$ and 
\begin{equation*}
    \inf_{h\in\hpc:P_{\mathcal{X}}\{x:h(x)\neq h^{*}(x)\}>0}P_{\mathcal{X}}\{x:h(x)\neq h^{*}(x)\} = 0 .
\end{equation*}
\end{lemma}

\begin{proof}[Proof of Lemma \ref{lem:condition-constant-gap-error-rate-implies-what-restated}]
Based on the assumption as in the lemma statement, we know that there exists a sequence of concepts $\{h_{i}\}_{i\geq1}\subseteq\hpc$ such that 
\begin{equation*}
    \text{er}_{P}(h_{i}) > \inf_{h^{\prime}\in\hpc}\text{er}_{P}(h^{\prime}) \;\;\text{ and }\;\; \text{er}_{P}(h_{i})\rightarrow\inf_{h^{\prime}\in\hpc}\text{er}_{P}(h^{\prime}) \text{ as } i\rightarrow\infty .
\end{equation*}
Since $\hpc$ is totally bounded in $L_{1}(P_{\mathcal{X}})$ pseudo-metric, every infinite sequence therein admits a Cauchy subsequence. Hence, there exists an increasing sequence of integers $\{i_{t}\}_{t\geq1}$ such that $\{h_{i_{t}}\}_{t\geq1}$ is a Cauchy sequence, i.e.
\begin{equation*}
    \lim_{t\rightarrow\infty}\sup_{t^{\prime}>t}P_{\mathcal{X}}\{x:h_{i_{t}}(x)\neq h_{i_{t^{\prime}}}(x)\} = 0 .
\end{equation*}
We simply choose $h^{*}$ such that $\lim_{t\rightarrow\infty}P_{\mathcal{X}}\{x:h_{i_{t}}(x)\neq h^{*}(x)\}=0$, which also satisfies
\begin{equation*}
    \trueerrorratehstar = \lim_{t\rightarrow\infty}\text{er}_{P}(h_{i_{t}}) = \lim_{i\rightarrow\infty}\text{er}_{P}(h_{i}) = \inf_{h^{\prime}\in\hpc}\text{er}_{P}(h^{\prime}) .
\end{equation*}
\end{proof}

\begin{lemma}  [\textbf{Lemma \ref{lem:target-dependent-superroot-upperbound-iff} restated}]
  \label{lem:target-dependent-superroot-upperbound-iff-restated}
Let $\hpc$ be any UGC concept class and $h^{*}$ be any classifier. Then $h^{*}$ is agnostically universally learnable at rate $o(n^{-1/2})$ by ERM if and only if the Condition \ref{cond:finite-vcd-for-sufficiently-small-ball} holds for $h^{*}$. 
\end{lemma}

\begin{proof}[Proof of Lemma \ref{lem:target-dependent-superroot-upperbound-iff-restated}]
We first prove the sufficiency. Suppose that the Condition \ref{cond:finite-vcd-for-sufficiently-small-ball} holds for $h^{*}$. For any distribution $P$ centered at $h^{*}$, there exists $\epsilon_{0}:=\epsilon_{0}(P)>0$ such that $\text{VC}(\hpc(\epsilon_{0}; P, h^{*}))<\infty$. Furthermore, since $\hpc$ is a UGC class, following the same argument in the proof of Lemma \ref{lem:target-specified-exponential-upperbound}, we know that there exists a distribution-dependent integer $n_{0}:=n_{0}(P)<\infty$ such that for any $n\geq n_{0}$, with probability of at least $1-4e^{-n\epsilon_{0}^{2}/8}$, $\sup_{h\in\hpc}|\empiricalerrorrateh-\trueerrorrateh|\leq\epsilon_{0}/2$. It follows that
\begin{equation*}
    \trueerrorratehn-\trueerrorratehstar \leq \empiricalerrorratehn-\min_{h\in\hpc}\empiricalerrorrateh+\epsilon_{0} = \epsilon_{0} 
\end{equation*}
holds on an event $E_{1}(n)$ with probability of at least $1-4e^{-n\epsilon_{0}^{2}/8}$ for any $n\geq n_{0}$. In other words, $\lalgo\in\hpc(\epsilon_{0}; P, h^{*})$, namely, any ERM algorithm will then output a hypothesis from a VC class. Applying the inequality $\mathbb{P}(E)\leq\mathbb{P}(E|F)+\mathbb{P}(\neg F)$, we have for any $n\geq n_{0}$,
\begin{align*}
    \excessrisk = &\int_{0}^{\infty}\mathbb{P}\left(\trueerrorratehn-\trueerrorratehstar>\epsilon\right)\mathrm{d}\epsilon \\
    \leq &\int_{0}^{\infty}\left[\mathbb{P}\left(\trueerrorratehn-\trueerrorratehstar>\epsilon \big| E_{1}(n)\right) + \mathbb{P}\left(\neg E_{1}(n)\right)\right]\mathrm{d}\epsilon \\
    = &\int_{0}^{\epsilon_{0}}\mathbb{P}\left(\trueerrorratehn-\trueerrorratehstar>\epsilon \big| E_{1}(n)\right)\mathrm{d}\epsilon + \int_{0}^{1}\mathbb{P}\left(\neg E_{1}(n)\right)\mathrm{d}\epsilon \\
    \leq &\int_{0}^{\epsilon_{0}}\mathbb{P}\left(\trueerrorratehn-\trueerrorratehstar>\epsilon \big| E_{1}(n)\right)\mathrm{d}\epsilon + 4e^{-n\epsilon_{0}^{2}/8} \\
    \leq &\E\left[\trueerrorratehn-\trueerrorratehstar \big| E_{1}(n)\right] + 4e^{-n\epsilon_{0}^{2}/8} \stackrel{\text{\tiny Lemma \ref{lem:super-root-upperbound}}}{=} o(n^{-1/2}) ,
\end{align*}
where the rate function is dependent on the distribution-dependent constants $\epsilon_{0}$, $\text{VC}(\hpc(\epsilon_{0}; P, h^{*}))$ and $n_{0}$.

To show the necessity, we assume to the contrary that the Condition \ref{cond:finite-vcd-for-sufficiently-small-ball} fails for $h^{*}$, i.e. there exists a distribution $P$ centered at $h^{*}$ such that for any $\epsilon>0$, $\text{VC}(\hpc(\epsilon; P, h^{*}))=\infty$. Following a similar idea of the proof for Lemma \ref{lem:no-infinite-eluder-sequence-implies-what}, we aim to show there exists an infinite VC-eluder sequence centered at $h^{*}$, which will give a contradiction in conjunction with Lemma \ref{lem:arbitrarily-slow-lowerbound}. Such an infinite VC-eluder sequence can be constructed via the following procedure:
\setlist{nolistsep}
\begin{itemize}
    \item[1.] Let $\epsilon_{1}>0$. We pick a point $x_{1}\in\text{DIS}(\hpc(\epsilon_{1}; P, h^{*}))$ shattered by $\hpc(\epsilon_{1}; P, h^{*})\subseteq\hpc$ such that
    \begin{equation}
      \label{eq:lem-target-dependent-superroot-upperbound-iff-intermediate-step1}
        \inf_{h\in\hpc(\epsilon_{1}; P, h^{*}):h(x_{1})\neq h^{*}(x_{1})}\left\{\trueerrorrateh-\trueerrorratehstar\right\} > 0 .
    \end{equation}
    Note that such an $x_{1}$ can actually be any point satisfying $\mathbb{P}(h^{*}(x)\neq y,x=x_{1})<1/2$.
    \item[2.] The next step is to find two points $x_{2},x_{3}\in\mathcal{X}$ shattered by $V_{(x_{1},h^{*}(x_{1}))}(\hpc(\epsilon_{1}; P, h^{*}))$, and it suffices to show that the version space also has infinite VC dimension. To this end, we show that there exists $0<\epsilon_{2}\leq\epsilon_{1}$ such that $\hpc(\epsilon_{2}; P, h^{*})\subseteq V_{(x_{1},h^{*}(x_{1}))}(\hpc(\epsilon_{1}; P, h^{*}))$, i.e. any $h\in\hpc(\epsilon_{2}; P, h^{*})$ satisfies $h(x_{1})=h^{*}(x_{1})$. This can be realized by setting $\epsilon_{2}$ to be the LHS of \eqref{eq:lem-target-dependent-superroot-upperbound-iff-intermediate-step1} since any $h\in\hpc(\epsilon_{1}; P, h^{*})$ satisfying $h(x_{1})\neq h^{*}(x_{1})$ will have an excess risk larger than $\epsilon_{2}$. Therefore, it follows $\infty=\text{VC}(\hpc(\epsilon_{2}; P, h^{*}))\leq\text{VC}(V_{(x_{1},h^{*}(x_{1}))}(\hpc(\epsilon_{1}; P, h^{*})))$, and thus such $x_{2},x_{3}$ must exist. Similarly, we can make the following happen:
    \begin{equation*}
        \inf_{h\in\hpc(\epsilon_{2}; P, h^{*}):\text{DIS}(\{h,h^{*}\})\cap\{x_{2},x_{3}\}\neq\emptyset}\left\{\trueerrorrateh-\trueerrorratehstar\right\} > 0 .
    \end{equation*}
    \item[3.] Find three points $x_{4},x_{5},x_{6}\in\mathcal{X}$ shattered by $V_{\{(x_{1},h^{*}(x_{1})),(x_{2},h^{*}(x_{2})),(x_{3},h^{*}(x_{3}))\}}(\hpc(\epsilon_{1}; P, h^{*}))$.
    \item[...]
\end{itemize}
According to the definition, $\{x_{1},x_{2},\ldots\}$ is an infinite VC-eluder sequence centered at $h^{*}$.
\end{proof}

\section{Missing proofs from Section \ref{sec:bayes-dependent-rates}}
  \label{sec:missing-proofs-bayes-dependent}

\begin{theorem}  [\textbf{Theorem \ref{thm:bayes-dependent-exponential-exact-rates} restated}]
  \label{thm:bayes-dependent-exponential-exact-rates-restated}
Let $\hpc$ be any concept class and $h$ be any classifier. Then for any distribution $P$ such that $h$ is a Bayes-optimal classifier with respect to $P$ (see Definition \ref{def:bayes-optimal-classifier}), $\hpc$ is agnostically universally learnable at rate $e^{-n}$ by ERM under $P$ if and only if $\hpc$ does not have an infinite eluder sequence centered at $h$.
\end{theorem}

\begin{proof}[Proof of Theorem \ref{thm:bayes-dependent-exponential-exact-rates-restated}]
We first prove the sufficiency. Suppose to the contrary that there is a distribution $P$ such that $h=\bayesoptimal$, $\hpc$ does not have an infinite eluder sequence centered at $h$, but is not agnostically universally learnable at exponential rate (under $P$). Lemma \ref{lem:target-specified-exponential-upperbound} tells us that the Condition \ref{cond:constant-gap-error-rate} fails (under $P$), and then Lemma \ref{lem:condition-constant-gap-error-rate-implies-what} guarantees that there exists an optimal function $h^{*}$, i.e. $\trueerrorratehstar=\inf_{h\in\hpc}\trueerrorrateh$ such that $\inf_{h\in\hpc:P_{\mathcal{X}}\{x:h(x)\neq h^{*}(x)\}>0}P_{\mathcal{X}}\{x:h(x)\neq h^{*}(x)\}=0$, and these two limits can be achieved by the same sequence of concepts $\{h_{i}\}_{i\geq1}\subseteq\hpc$, that is,
\begin{align}
    &\lim_{i\rightarrow\infty}\text{er}_{P}(h_{i}) = \trueerrorratehstar, \;\; \lim_{i\rightarrow\infty}P_{\mathcal{X}}\{x:h_{i}(x)\neq h^{*}(x)\}=0 \label{eq:lemma-bayes-specified-exponential-upperbound-intermediate-step1} \\
    &\text{er}_{P}(h_{i})>\trueerrorratehstar, \;\; P_{\mathcal{X}}\{x:h_{i}(x)\neq h^{*}(x)\}>0, \;\;\; \forall i\in\naturalnumber \label{eq:lemma-bayes-specified-exponential-upperbound-intermediate-step2}
\end{align}
Furthermore, according to the second part proof of Lemma \ref{lem:no-infinite-eluder-sequence-implies-what}, we know there exists an infinite eluder sequence $\{(x_{t},h^{*}(x_{t}))\}_{t\geq1}$ centered at $h^{*}$ witnessed by some subsequence $\{h_{i_{t}}\}_{t\geq1}\subseteq\{h_{i}\}_{i\geq1}$. 

We first assume that $P$ is supported on the aforementioned infinite eluder sequence. Note that for these eluder points, $h(x_{t})=h^{*}(x_{t})$ cannot happen for infinitely many $t\in\naturalnumber$, since otherwise $h$ will also admit an infinite eluder sequence. We assume there are $k<\infty$ many of such eluder points, denoted by $\{x_{t_{1}},\ldots,x_{t_{k}}\}$, and consider a new infinite eluder sequence $\{(x_{t_{j}},h^{*}(x_{t_{j}}))\}_{j\geq1}$ centered at $h^{*}$ witnessed by $\{h_{i_{t_{j}}}\}_{j\geq1}\subseteq\{h_{i_{t}}\}_{t\geq1}$. Now since $h=\bayesoptimal$, any concept $h_{i_{t_{j}}}$ with $j>k$ satisfies $\text{er}_{P}(h_{i_{t_{j}}})<\trueerrorratehstar$ because $h^{*}$ disagrees with the Bayes-optimal classifier on every eluder point with index larger than $k$. This contradicts our assumption that $\trueerrorratehstar=\inf_{h\in\hpc}\trueerrorrateh$. Finally, if $P$ is not simply supported on the eluder sequence, we consider the following conditional distribution $P^{\prime}:=P(\cdot|\{x_{t}\}_{t\geq1})$. Note that $h$ is still the Bayes-optimal classifier since we have the same $P(\cdot|\mathcal{X})$. Moreover, since $\{h_{i_{t}}\}_{t\geq1}\subseteq\{h_{i}\}_{i\geq1}$, we still have $\inf_{h\in\hpc}\text{er}_{P^{\prime}}(h)=\text{er}_{P^{\prime}}(h^{*})$ and $\inf_{h\in\hpc:P_{\mathcal{X}}^{\prime}\{x:h(x)\neq h^{*}(x)\}>0}P_{\mathcal{X}}^{\prime}\{x:h(x)\neq h^{*}(x)\}=0$ based on \eqref{eq:lemma-bayes-specified-exponential-upperbound-intermediate-step1} and \eqref{eq:lemma-bayes-specified-exponential-upperbound-intermediate-step2}. Therefore, the previous argument can be applied on $P^{\prime}$. 

To show the necessity, we suppose to the contrary that $\hpc$ has an infinite eluder sequence centered at $h=\bayesoptimal$. Consider the distribution $P$ as constructed in Lemma \ref{lem:super-root-lowerbound}, for which $h$ is not only the Bayes-optimal classifier, but also satisfies $\inf_{h^{\prime}\in\hpc}\text{er}_{P}(h^{\prime})=\trueerrorrateh$. Therefore, it is not agnostically universally learnable at rate faster than $o(n^{-1/2})$ by ERM, which gives a contradiction!
\end{proof}

\begin{lemma}  
  \label{lem:bayes-specified-superroot-upperbound}
Let $\hpc$ be any concept class and $h$ be any classifier. Then for any distribution $P$ such that $h$ is a Bayes-optimal classifier with respect to $P$, $\hpc$ is agnostically universally learnable at rate $o(n^{-1/2})$ by ERM under $P$ if and only if $\hpc$ does not have an infinite VC-eluder sequence centered at $h$.
\end{lemma}

\begin{proof}[Proof of Lemma \ref{lem:bayes-specified-superroot-upperbound}]
The necessity is straightforward: if $\hpc$ has an infinite VC-eluder sequence centered at $h$, then the distribution $P$ constructed in the proof of Lemma \ref{lem:arbitrarily-slow-lowerbound} is centered at $h$ and has $h$ be a Bayes-optimal classifier. Therefore, we know that $\hpc$ requires at least arbitrarily slow rates to be agnostically universally learned by ERM under distribution $P$. 

To show the sufficiency, we suppose to the contrary that $\hpc$ is not agnostically universally learnable at rate $o(n^{-1/2})$ under some distribution $P$ having $h$ as a Bayes-optimal classifier, then Lemma \ref{lem:target-dependent-superroot-upperbound-iff} implies that the Condition \ref{cond:finite-vcd-for-sufficiently-small-ball} fails for $P$, and furthermore, there exists an infinite VC-eluder sequence $\{x_{1},x_{2},\ldots\}$ centered at some target function $h^{*}$ satisfying $\trueerrorratehstar=\inf_{h\in\hpc}\trueerrorrateh$. We note that an infinite VC-eluder sequence is an infinite eluder sequence, and it is impossible for a distribution $P$ having a target function realize an infinite eluder sequence but a Bayes-optimal classifier does not, according to Theorem \ref{thm:bayes-dependent-exponential-exact-rates}. Therefore, a contradiction follows.
\end{proof}

\begin{theorem}  [\textbf{Theorem \ref{thm:bayes-dependent-superroot-exact-rates} restated}]
  \label{thm:bayes-dependent-superroot-exact-rates-restated}
Let $\hpc$ be any concept class and $h$ be any classifier. For any distribution $P$ such that $h$ is a Bayes-optimal classifier with respect to $P$, $\hpc$ is agnostically universally learnable at exact rate $o(n^{-1/2})$ by ERM under $P$ if and only if $\hpc$ does not have an infinite VC-eluder sequence, but has an infinite eluder sequence centered at $h$.
\end{theorem}

\begin{theorem}  [\textbf{Theorem \ref{thm:bayes-dependent-arbitrarily-slow-rates} restated}]
  \label{thm:bayes-dependent-arbitrarily-slow-rates-restated}
Let $\hpc$ be any concept class and $h$ be any classifier. For any distribution $P$ such that $h$ is a Bayes-optimal classifier with respect to $P$, $\hpc$ requires at least arbitrarily slow rates to be agnostically universally learned by ERM under $P$ if and only if $\hpc$ has an infinite VC-eluder sequence centered at $h$.
\end{theorem}

\begin{proof} [Proof of Theorems \ref{thm:bayes-dependent-superroot-exact-rates-restated} and \ref{thm:bayes-dependent-arbitrarily-slow-rates-restated}]
We will prove Theorem \ref{thm:bayes-dependent-superroot-exact-rates-restated} and then the proof of Theorem~\ref{thm:bayes-dependent-arbitrarily-slow-rates-restated} is straightforward. To show the necessity, we note that Theorem~\ref{thm:bayes-dependent-exponential-exact-rates-restated} implies that ``an infinite eluder sequence centered at $h$" is necessary, since otherwise an exponential rate is attainable. Moreover, we consider the distribution $P$ as constructed in Lemma \ref{lem:arbitrarily-slow-lowerbound}, it is not hard to show that $h$ is not only the Bayes-optimal classifier, but also a target with respect to $P$. Therefore, if $\hpc$ has as infinite VC-eluder sequence centered at $h$, then it requires at least arbitrarily slow universal rates according to Lemma~\ref{lem:arbitrarily-slow-lowerbound}. This implies that the necessary condition also has to include ``no infinite VC-eluder sequence centered at $h$". To show the sufficiency, we have to prove an upper bound as well as a lower bound for proving an $o(n^{-1/2})$ exact rate. Note that Lemma~\ref{lem:bayes-specified-superroot-upperbound} implies the upper bound. The lower bound follows from the proof of Theorem~\ref{thm:bayes-dependent-exponential-exact-rates-restated}, where we show an infinite eluder sequence centered at $h$ yields a no faster than $o(n^{-1/2})$ universal rate.
\end{proof}

\section{Technical Lemmas}
  \label{sec:technical-lemmas}

\begin{lemma}  [\textbf{Hoeffding's Inequality}]
  \label{lem:hoeffding-ineq}
Let $Z_{1},\ldots,Z_{n}$ be independent random variables in $[a,b]$ with some constants $a<b$, let $\bar{Z} := \frac{1}{n}\sum_{i=1}^{n}Z_{i}$. Then for any $t > 0$,
\begin{equation*}
\mathbb{P}\left(\big|\bar{Z} - \E[\bar{Z}]\big| > t\right) \leq 2e^{-\frac{2nt^{2}}{(b-a)^{2}}} .
\end{equation*}
\end{lemma}

\begin{lemma}  [\textbf{McDiarmid’s Inequality}]
  \label{lem:mcdiarmid-ineq}
Let $V$ be some set and let $f:V^{n}\rightarrow\R$ be a function of $n$ variables such that for some $c>0$, for all $i\in[n]$ and for all $x_{1},\ldots,x_{n},x_{i}^{\prime}\in V$, we have the following ``bounded differences property":
\begin{equation*}
    \big|f(x_{1},\ldots,x_{n})-f(x_{1},\ldots,x_{i-1},x_{i}^{\prime},x_{i+1},\ldots,x_{n})\big| \leq c .
\end{equation*}
Then let $X_{1},\ldots,X_{n}$ be independent random variables taking values in $V$, we have
\begin{equation*}
    \Big|f(X_{1},\ldots,X_{n}) - \E\left[f(X_{1},\ldots,X_{n})\right]\Big| \leq c\sqrt{\frac{n}{2}\ln{\left(\frac{2}{\delta}\right)}} ,
\end{equation*}
with probability at least $1-\delta$.
\end{lemma}

\begin{lemma}
  \label{lem:integral-help-lemma1}
For any $a,x>0$,
\begin{equation*}
    \int_{x}^{\infty}e^{-a\epsilon^{2}}d\epsilon \leq \int_{x}^{\infty}\frac{\epsilon}{x}e^{-a\epsilon^{2}}d\epsilon = \frac{1}{2ax}e^{-ax^{2}} .
\end{equation*}
\end{lemma}

\begin{lemma}  [{\textbf{\citealp[][]{slud1977distribution}}}]
  \label{lem:binomial-bound}
Let $X\sim\text{Binomial}(n,p)$ and assume that $p=(1-\epsilon)/2$, then
\begin{equation*}
    \mathbb{P}\left(X\geq\frac{n}{2}\right) \geq \frac{1}{2}\left(1-\sqrt{1-\exp{(-n\epsilon^{2}/(1-\epsilon^{2}))}}\right) . 
\end{equation*}
\end{lemma}

\begin{lemma}  [{\textbf{\citealp[][Lemma 5.12]{bousquet2021theory}}}]
  \label{lem:infinite-sequence-design2}
For any function $R(n)\rightarrow0$, there exist probabilities $\{p_{t}\}_{t\in\naturalnumber}$ satisfying $\sum_{t\geq1}p_{t}=1$, two increasing sequences of integers $\{n_{t}\}_{t\in\naturalnumber}$ and $\{k_{t}\}_{t\in\naturalnumber}$, and a constant $1/2\leq C\leq 1$ such that the following hold for all $t\in\naturalnumber$:
\begin{itemize}
    \item[(1)] $\sum_{k>k_{t}}p_{k} \leq \frac{1}{n_{t}}$.
    \item[(2)] $n_{t}p_{k_{t}} \leq k_{t}$.
    \item[(3)] $p_{k_{t}} = CR(n_{t})$.
\end{itemize}
\end{lemma}

\begin{lemma}
  \label{lem:infinite-sequence-design1}
Under the setting of Lemma~\ref{lem:infinite-sequence-design2}, the following holds
\begin{equation*}
    \sum_{j>t}\frac{p_{j}}{\sqrt{n_{j}}} \leq \frac{p_{t}}{\sqrt{n_{t}}}
\end{equation*}
\end{lemma}

\begin{proof}[Proof of Lemma \ref{lem:infinite-sequence-design1}]
Recall the construction of the sequences in the proof of Lemma~\ref{lem:infinite-sequence-design2}: let $R(1)=1$, and the sequences $\{n_{t}\}_{t\in\naturalnumber}$ and $\{k_{t}\}_{t\in\naturalnumber}$ are defined recursively. Let $n_{1}=1$ and $k_{1}=1$, for $t>1$, define
\begin{equation*}
    n_{t} := \inf\left\{n>n_{t-1}: R(n)\leq\min_{j<t}\left\{\frac{R(n_{j})\cdot2^{j-t}}{k_{t}}\right\}\right\} .
\end{equation*}
Moreover, we define $p_{t}=CR(n_{t})$ for all $t\in\naturalnumber$, where the constant $C=1/(\sum_{t\geq1}R(n_{t}))$. Note that $R(n_{t})\leq 2^{-t+1}$ for all $t>1$, it holds that $1/2\leq C\leq 1$. With such a construction, we have
\begin{equation*}
    \sum_{j>t}\frac{p_{j}}{\sqrt{n_{j}}} = \sum_{j>t}\frac{CR(n_{j})}{\sqrt{n_{j}}} \leq \sum_{j>t}\frac{CR(n_{j})}{\sqrt{n_{t}}} \leq \sum_{j>t}\frac{CR(n_{t})2^{t-j}}{k_{j}\sqrt{n_{t}}} \leq \frac{p_{t}}{\sqrt{n_{t}}}.
\end{equation*}
\end{proof}

\begin{lemma}
  \label{lem:distance-empiricalerror-and-trueerror}
Let $\hpc$ be a concept class and $P$ be any distribution, let $S_{n}\sim P^{n}$ for any $n\in\naturalnumber$. For any $\delta\in(0,1)$, with probability of at least $1-\delta$ we have
    \begin{equation*}
        \sup_{h\in\hpc}\Big|\empiricalerrorrateh-\trueerrorrateh\Big| \leq \E\left[\sup_{h\in\hpc}\Big|\empiricalerrorrateh-\trueerrorrateh\Big|\right] + \sqrt{\frac{1}{2n}\ln{\left(\frac{4}{\delta}\right)}} .
    \end{equation*}
\end{lemma}

\begin{proof}[Proof of Lemma \ref{lem:distance-empiricalerror-and-trueerror}]
Note that 
\begin{equation*}
    \sup_{h\in\hpc}\Big|\empiricalerrorrateh-\trueerrorrateh\Big| = \max\left\{\sup_{h\in\hpc}\empiricalerrorrateh-\trueerrorrateh, \sup_{h\in\hpc}\trueerrorrateh-\empiricalerrorrateh\right\} .
\end{equation*}
We will consider the first term $\sup_{h\in\hpc}\empiricalerrorrateh-\trueerrorrateh$ below, and the other can be analyzed similarly. Let $S_{n}:=\{(x_{i},y_{i})\}_{i=1}^{n}$ and $S_{n,i}^{\prime}:=\{(x_{1},y_{1}),\ldots,(x_{i-1},y_{i-1}),(x_{i}^{\prime},y_{i}^{\prime}),(x_{i+1},y_{i+1}),\ldots,(x_{n},y_{n})\}$ for any $(x_{i}^{\prime},y_{i}^{\prime})\sim P$, we have the following ``bounded differences property":
\begin{align*}
    &\Big|\sup_{h\in\hpc}\empiricalerrorrateh-\trueerrorrateh - \sup_{h\in\hpc}\hat{\text{er}}_{S_{n,i}^{\prime}}(h)-\trueerrorrateh\Big| \\
    \leq &\sup_{h\in\hpc}\big|(\empiricalerrorrateh-\trueerrorrateh) - (\hat{\text{er}}_{S_{n,i}^{\prime}}(h)-\trueerrorrateh)\big| \\
    = &\sup_{h\in\hpc}\bigg|\frac{1}{n}\left(\mathbbm{1}\{h(x_{i})\neq y_{i}\}-\mathbbm{1}\{h(x_{i}^{\prime})\neq y_{i}^{\prime}\}\right)\bigg| \leq \frac{1}{n} .
\end{align*}
By applying the McDiarmid's inequality (Lemma \ref{lem:mcdiarmid-ineq}), we have 
\begin{equation*}
    \sup_{h\in\hpc}\empiricalerrorrateh-\trueerrorrateh \leq \E\left[\sup_{h\in\hpc}\empiricalerrorrateh-\trueerrorrateh\right] + \sqrt{\frac{1}{2n}\ln{\left(\frac{4}{\delta}\right)}} ,
\end{equation*}
with probability at least $1-\delta/2$. The conclusion follows immediately from a union bound and the following fact:
\begin{equation*}
    \max\left\{\E\left[\sup_{h\in\hpc}\empiricalerrorrateh-\trueerrorrateh\right], \E\left[\sup_{h\in\hpc}\trueerrorrateh-\empiricalerrorrateh\right]\right\} \leq \E\left[\sup_{h\in\hpc}\Big|\empiricalerrorrateh-\trueerrorrateh\Big|\right] .
\end{equation*}
\end{proof}

\begin{lemma}  [{\textbf{\citealp[][Theorem 2.3]{bousquet2002bennett}}}]
  \label{lem:Bennet-concentration}
Let $P$ be a distribution and $\{X_{i}\}_{1\leq i\leq n}\subseteq\mathcal{X}$ be a sequence of i.i.d. $P$-distributed random variables. Let $\mathcal{F}$ be a countable set of functions from $\mathcal{X}$ to $\R$ and assume that all functions $f$ in $\mathcal{F}$ are $P$-measurable, square-integrable and satisfy $\E[f]=0$. We assume that $\sup_{f\in\mathcal{F}}\lVert f\rVert_{\infty}\leq 1$ and denote
\begin{equation*}
    Z := \sup_{f\in\mathcal{F}}\bigg|\sum_{i=1}^{n}f(X_{i})\bigg| .
\end{equation*}
Let $\sigma$ be a positive real number such that $\sup_{f\in\mathcal{F}}\mathrm{Var}[f(X_{1})]\leq\sigma^{2}$ almost surely, let $\nu=n\sigma^{2}+2\E[Z]$. Then for all $x\geq0$, we have
\begin{equation*}
    \mathbb{P}\left(Z\geq\E[Z]+x\right) \leq \exp{\left(-\nu h(x/\nu)\right)} ,
\end{equation*}
where $h(x):=(1+x)\log{(1+x)}-x$, and also
\begin{equation*}
    \mathbb{P}\left(Z\geq\E[Z]+\sqrt{2\nu x}+\frac{x}{3}\right) \leq e^{-x} .
\end{equation*}
\end{lemma}

\begin{lemma}  [{\textbf{\citealp[][Theorem 2.1]{van-der-Vaart11}}}]
  \label{lem:entropy-uniform-bound}
Let $\mathcal{F}$ be a $P$-measurable class of measurable functions with
envelope function $F\leq1$ and such that $\mathcal{F}^{2}$ is $P$-measurable. If $Pf^{2}<\delta^{2}PF^{2}$ for every $f$ and some $\delta\in(0,1)$, then 
\begin{equation*}
    \E_{P}^{*}\left[\sup_{f\in\mathcal{F}}\bigg|\frac{1}{\sqrt{n}}\sum_{i=1}^{n}\left(f(X_{i})-Pf\right)\bigg|\right] \lesssim J\left(\delta,\mathcal{F},L_{2}\right)\left(1+\frac{J\left(\delta,\mathcal{F},L_{2}\right)}{\delta^{2}\sqrt{n}\lVert F\rVert_{P,2}}\right)\lVert F\rVert_{P,2} ,
\end{equation*}
where $\lVert f\rVert_{P,2}$ denotes the norm of a function $f$ in $L_{2}(P)$ and
\begin{equation*}
    J\left(\delta,\mathcal{F},L_{2}\right) := \sup_{P}\int_{0}^{\delta}\sqrt{1+\log N(\epsilon\lVert F\rVert_{P,2},\mathcal{F},L_{2}(P))} \mathrm{d}\epsilon .
\end{equation*}
\end{lemma}

\end{document}